\documentclass{article}

\usepackage[numbers]{natbib}
\usepackage[preprint]{neurips_2026}

\usepackage[utf8]{inputenc} 
\usepackage[T1]{fontenc}    
\usepackage{hyperref}       
\usepackage{url}            
\usepackage{booktabs}       
\usepackage{amsfonts}       
\usepackage{nicefrac}       
\usepackage{microtype}      
\usepackage{xcolor}         

\usepackage{enumitem}
\usepackage{graphicx}
\usepackage{booktabs}
\usepackage{multirow}
\usepackage{diagbox}
\usepackage{makecell}
\usepackage{array}
\usepackage{arydshln}
\usepackage{hhline}
\usepackage{tabularx}
\usepackage{wrapfig}
\usepackage{colortbl}
\usepackage{algorithm}
\usepackage{algpseudocode}
\usepackage{adjustbox}
\usepackage{float}
\usepackage{amsmath, amssymb, amsthm, amsfonts}
\usepackage[capitalize,noabbrev]{cleveref}
\usepackage{subcaption}
\usepackage{dsfont}
\usepackage[table]{xcolor}

\newtheorem{hypothesis}{Hypothesis}

\newtheorem{proposition}{Proposition}

\usepackage{algorithm}
\definecolor{commentcolor}{RGB}{110,154,155}   

\title{Self-Denoising: When Explanations Denoise Themselves in Self-Interpretable GNNs}
\title{When Explanations Denoise Themselves in Self-Interpretable GNNs}
\title{Why Self-Inconsistency Arises in GNN Explanations and How to Exploit It}

\author{%
  Wenxin Tai, Yaqian Liu, Ting Zhong, Fan Zhou\thanks{Corresponding author.} \\
  University of Electronic Science and Technology of China \\
}

\begin{document}

\maketitle

\begin{abstract}
    Recent work has observed that explanations produced by Self-Interpretable Graph Neural Networks (SI-GNNs) can be self-inconsistent: when the model is reapplied to its own explanatory graph subset, it may produce a different explanation. However, why self-inconsistency arises remains poorly understood. In this work, we first identify re-explanation-induced context perturbation as the direct cause of score variation. We then introduce a latent signal assignment hypothesis to explain why only some edges are sensitive to this perturbation, and analyze how conciseness regularization affects latent signal assignment. Given that self-inconsistent edges do not provide stable evidence for the model's prediction, we propose \textit{Self-Denoising (SD)}, a model-agnostic and training-free post-processing strategy that calibrates explanations with only one additional forward pass. Experiments across representative SI-GNN frameworks, backbone architectures, and benchmark datasets support our hypothesis and show that SD consistently improves explanation quality while adding only about 4--6\% computational overhead in practice. 
\end{abstract}

\begin{figure}[ht]    
    \centering    
    \includegraphics[width=\linewidth]{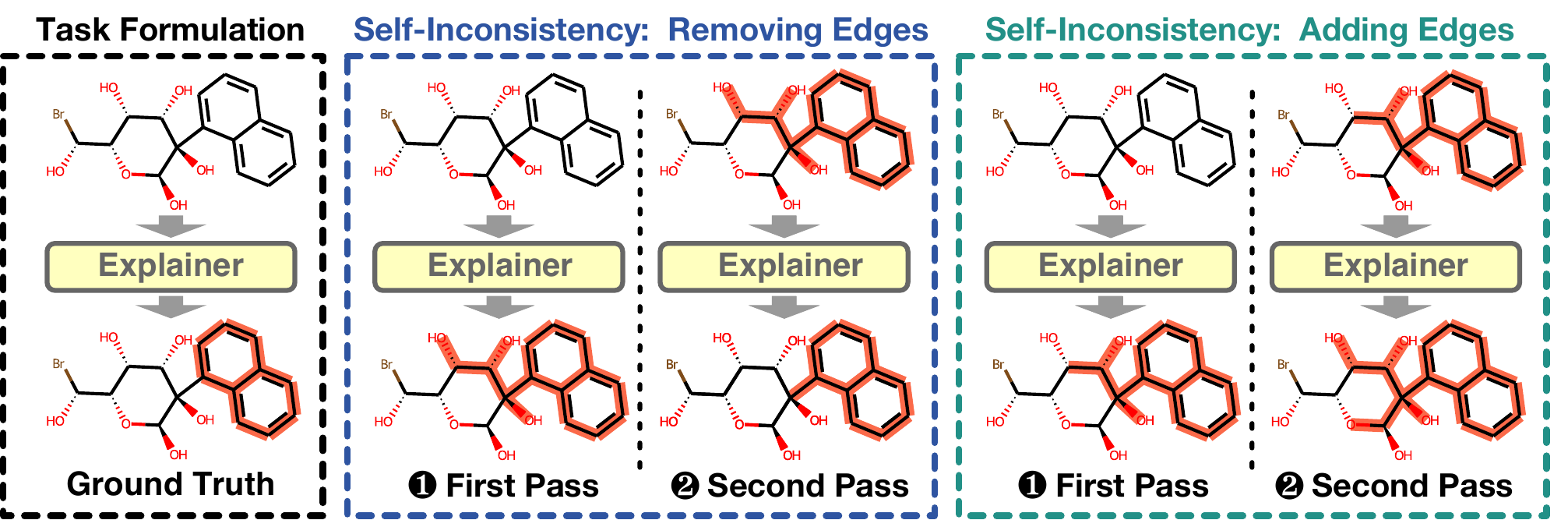}    
    \caption{Illustration of self-inconsistency in SI-GNNs, where re-explanation may remove selected edges or add new ones. This issue has been observed in all representative SI-GNN frameworks~\cite{tai2026self}.} 
    \label{fig:intro}
    \vspace{-0.2cm}
\end{figure}

\section{Introduction}

Self-Interpretable Graph Neural Networks (SI-GNNs)~\cite{velivckovic2018graph,sui2022causal,azzolin2025beyond,miao2022interpretable} produce explanations as part of prediction: they first identify an explanatory graph subset and then make predictions based on this subset. Unlike post-hoc explainers~\cite{ying2019gnnexplainer,luo2020parameterized,luo2024towards} that interpret a pretrained model from the outside, SI-GNNs provide built-in explanations, which are therefore generally considered more trustworthy~\cite{rudin2019stop,ragodos2024model}. However, recent work~\cite{tai2026self} shows that such trust is not always warranted. As illustrated in \cref{fig:intro}, when an SI-GNN is reapplied to its own explanatory graph subset, the newly produced explanation may differ from the original one, either removing previously selected edges or adding new ones. Such contradictory behavior, known as self-inconsistency, weakens the faithfulness and practical utility of SI-GNN explanations~\cite{tai2026self}, yet why it arises remains poorly understood.

In this work, we revisit the message-passing mechanism of GNNs and show that re-explanation perturbs local message-passing contexts, which is the direct cause of self-inconsistency (\cref{subsec:local-context}). To explain why some edges remain stable, especially important edges defined by ground-truth explanations, we introduce a latent signal assignment hypothesis (\cref{subsec:latent-signal}): the model assigns latent signals to edges it deems important, and these signals determine their edge scores. In contrast, edges without such signals are more susceptible to contextual perturbations. We further analyze how conciseness regularization shapes signal allocation through the explanation budget, with theoretical analysis and empirical observations supporting this perspective (\cref{subsec:budget-signal}).

Based on this hypothesis, we propose \textit{Self-Denoising (SD)}, a model-agnostic, simple, and efficient post-processing strategy to calibrate explanations (\cref{subsec:sd}). SD penalizes edges whose importance varies across repeated explanations, aiming to suppress context-driven noise while preserving signal-supported structures. We further discuss when SD improve explanation quality and what side effects it may introduce (\cref{subsec:ood}), and provide a practical strategy for selecting its denoising strength (\cref{subsec:eta-select}). Extensive experiments across representative SI-GNN frameworks, GNN backbones, and benchmark datasets demonstrate that SD consistently improves explanation quality. Moreover, its complementarity with Explanation Ensemble (EE)~\cite{tai2025redundancy}, an existing post-hoc explanation calibration strategy, suggests that self-inconsistency captures a distinct source of explanation noise. 

Code will be publicly released in a future version.

\section{Preliminaries}

In this section, we introduce SI-GNNs and the notion of self-inconsistency.

\subsection{Self-Interpretable Graph Neural Networks (SI-GNNs)}

A graph is denoted as $G = (\mathcal{V}, \mathcal{E})$, where $\mathcal{V}$ and $\mathcal{E}$ represent the sets of nodes and edges, respectively. 
For graph classification tasks, a standard GNN can be viewed as a composition of two functional components: 
a \textit{GNN encoder} $h_Z: G \rightarrow \mathbb{R}^d$ that maps the input graph to a latent representation, 
and a \textit{classifier} $h_{\hat{Y}}: \mathbb{R}^d \rightarrow \mathbb{R}^c$ that produces the final prediction. 
The overall mapping can be expressed as $f = h_{\hat{Y}} \circ h_Z$. SI-GNNs extend this formulation by introducing an \textit{explainer} $h_{G_s}$, which assigns edge importance scores through a soft edge mask $M \in [0,1]^{|\mathcal{E}|}$. 
Each element in $M$ indicates the importance of an edge for the prediction. 
The explanatory graph is then represented as $G_s = G \odot M$, where $\odot$ denotes edge-wise masking. The predictive process of SI-GNNs can be reformulated as $f = h_{\hat{Y}} \circ h_Z \circ h_{G_s}$. Following prior studies~\citep{tai2025redundancy,tai2026self}, we focus on instance-level explanations that emphasize the importance of structural patterns, i.e., edge-level explanations. Given a labeled sample $(G,Y)$, the objective of SI-GNNs is to learn an explanatory graph $G_s$ that preserves label-relevant information while remaining concise. 
This objective is commonly formulated as a mutual information (MI) maximization problem with a conciseness regularization term:
\begin{align}
    \max_{G_s} \; I(G_s; Y) - \beta R(G_s),
\end{align}
where $I(\cdot;\cdot)$ denotes mutual information, 
$R(G_s)$ enforces explanation conciseness, 
and $\beta$ controls the strength of conciseness regularization. 
Different SI-GNNs implement $R(G_s)$ in different ways, and we review the most common ones below~\cite{tai2025redundancy}.

\textbf{(1) Attention-based methods} assign edge importance scores through attention coefficients~\cite{vaswani2017attention}. A representative example is GAT~\cite{velivckovic2018graph}, which is trained solely with the classification loss:
\begin{align} \label{eq:type_1}
    \mathcal{L}_{\mathrm{GE}} = \mathcal{L}_{\mathrm{CE}}(Y, \hat{Y}|G_s).
\end{align}

\textbf{(2) Causal-based methods} use causal inference~\cite{pearl2014interpretation} to identify causal structures beyond spurious correlations~\cite{wu2022discovering}. 
A representative example is CAL~\cite{sui2022causal}, which uses disentanglement and intervention to assess whether selected edges are causally responsible for the model's output:
\begin{align} \label{eq:type_2}
    \mathcal{L}_{\mathrm{GE}} = \mathcal{L}_{\mathrm{CE}}(Y, \hat{Y}|G_s) + \beta \cdot \mathbb{D}_{\mathrm{KL}}(\mathbb{P}_{\boldsymbol{\theta}}(\bar{Y} | \bar{G}_s) || \mathbb{Q}(\bar{Y})) + \gamma \cdot \mathcal{L}_{\mathrm{CE}}(Y, \hat{Y}' | G_s \cup \bar{G}_s'),
\end{align}
where $\bar{G}_s = G \setminus G_s$ denotes the complementary subgraph, and $\bar{G}_s'$ is obtained via intervention (e.g., replacing latent representations with those from other samples in the same batch). 
The auxiliary prior distribution $\mathbb{Q}(\bar{Y})$ is typically chosen as uniform.

\textbf{(3) Size-constrained methods} encourage concise explanations by penalizing the size of the selected subgraph~\cite{lin2020graph,luo2024towards}. A representative example is SMGNN~\cite{azzolin2025beyond}, which uses a sparsity regularizer:
\begin{align} \label{eq:type_3}
    \mathcal{L}_{\mathrm{GE}} = \mathcal{L}_{\mathrm{CE}}(Y, \hat{Y}|G_s) + \beta \cdot \frac{|G_s|}{|G|},
\end{align}
where $|G_s|$ is the total edge-mask mass, and $|G|$ is the number of edges in the original graph.

\textbf{(4) MI-constrained methods} adopt an information-theoretic perspective, aiming to extract a minimal yet sufficient subgraph by limiting the MI between the explanation and the original graph~\cite{yu2021graph,yu2022improving}. A representative example is GSAT~\cite{miao2022interpretable}, which approximates MI minimization via KL divergence:
\begin{align} \label{eq:type_4}
    \mathcal{L}_{\mathrm{GE}} = \mathcal{L}_{\mathrm{CE}}(Y, \hat{Y}|G_s) + \beta \cdot \mathbb{D}_{\mathrm{KL}}\big(\mathbb{P}_{\boldsymbol{\theta}}(G_s|G) || \mathbb{Q}(G_s)\big),
\end{align}
where $\mathbb{Q}(G_s)$ is typically defined as a Bernoulli prior. Other works are discussed in \cref{app:related}.

\subsection{Explanation Self-Consistency}

\begin{figure}[t]  
\centering  
\begin{subfigure}{0.24\linewidth}    
\includegraphics[width=\textwidth]{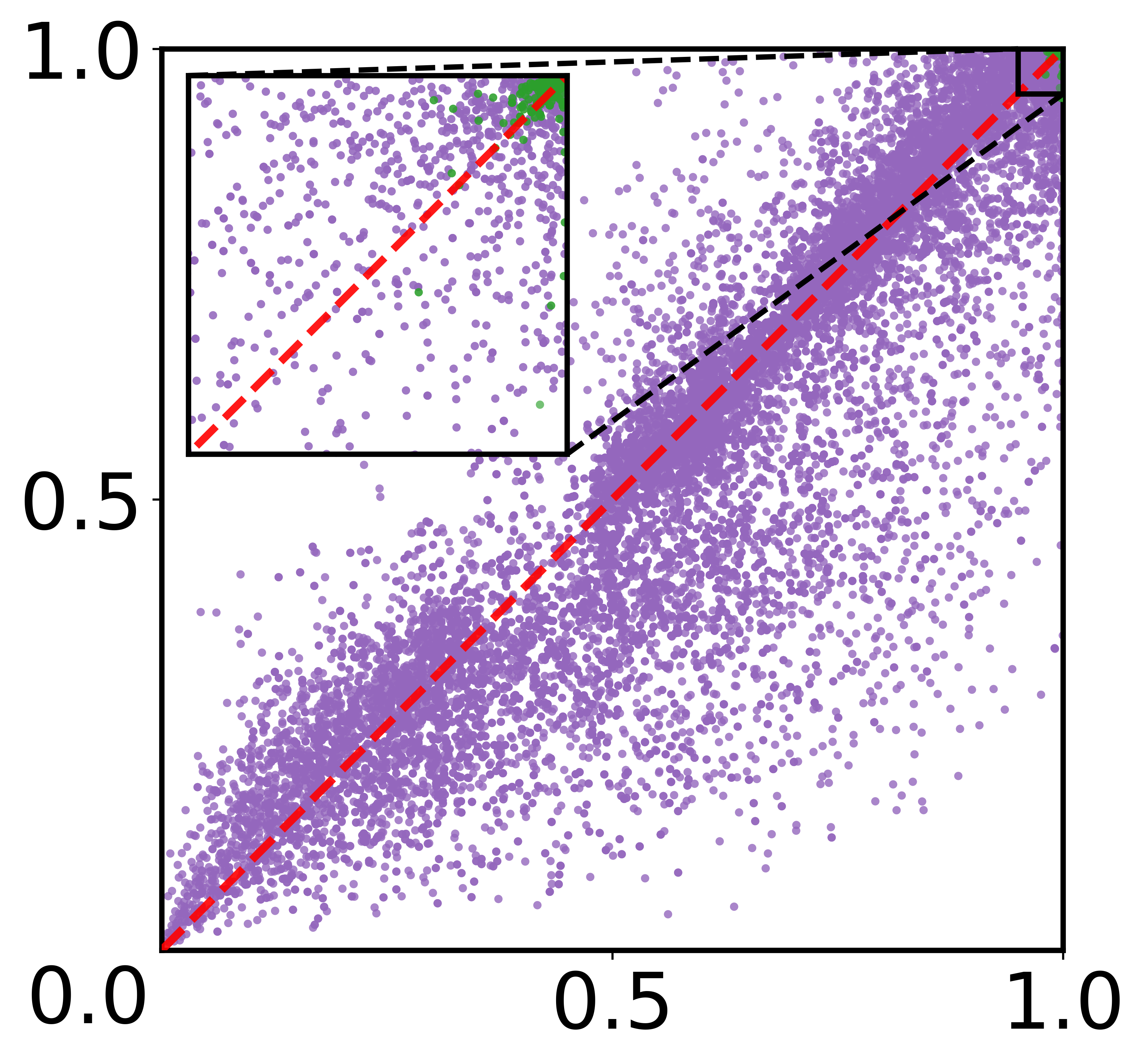}    
\caption{GAT}  
\end{subfigure}  
\hfill  
\begin{subfigure}{0.24\linewidth}    
\includegraphics[width=\textwidth]{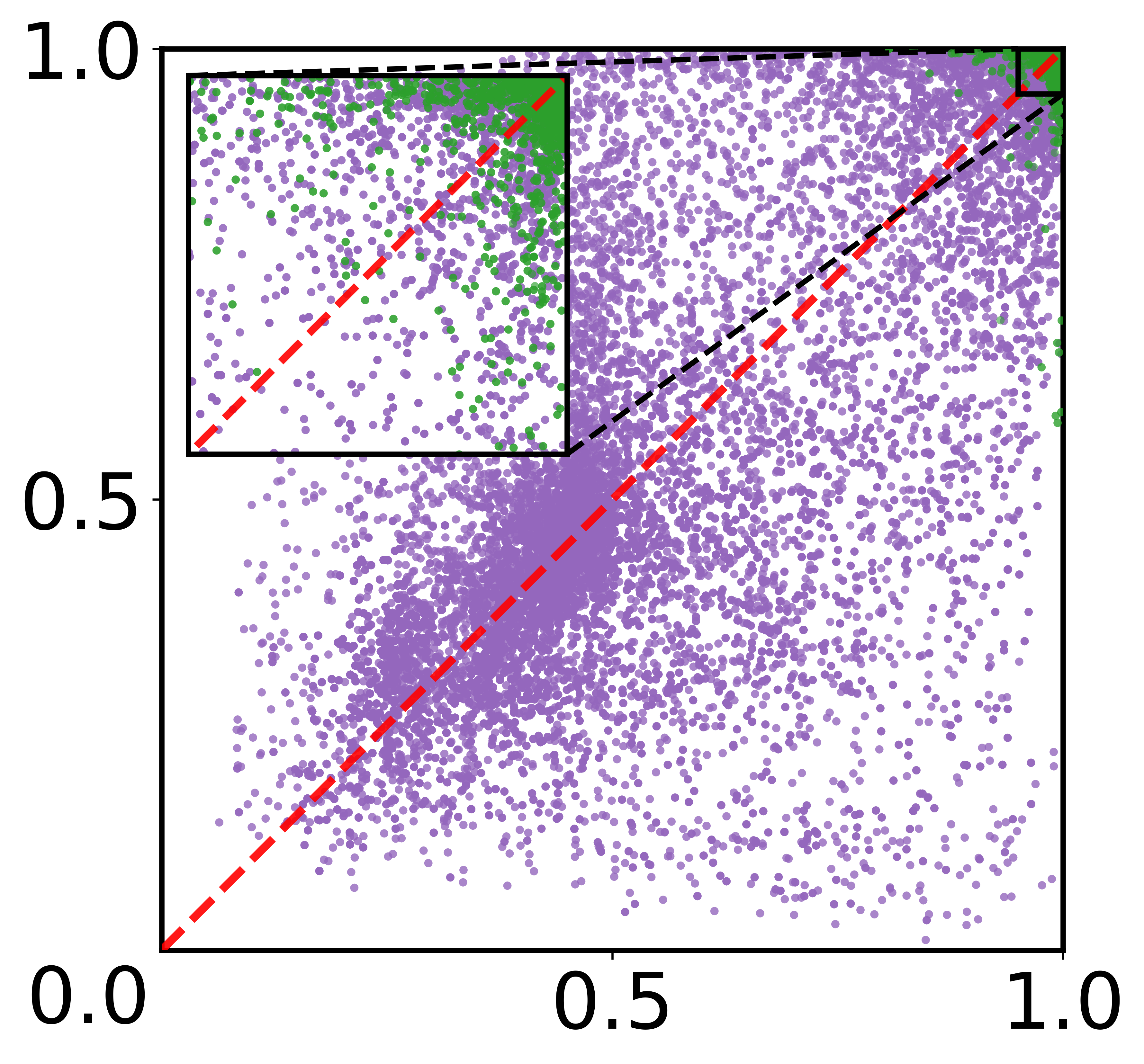}    
\caption{CAL}  
\end{subfigure}  
\hfill  
\begin{subfigure}{0.24\linewidth}    
\includegraphics[width=\textwidth]{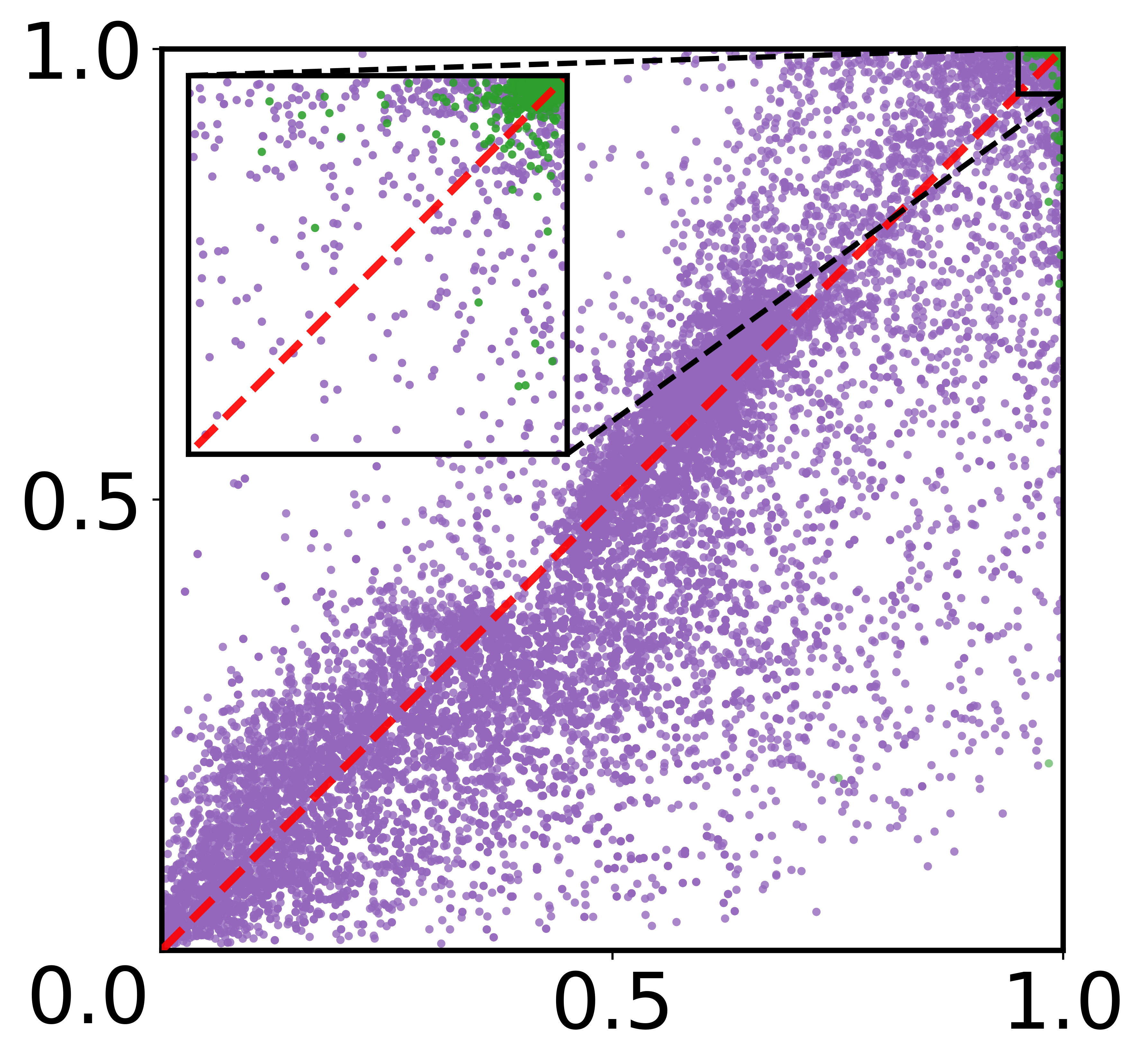}    
\caption{SMGNN}  
\end{subfigure}  
\hfill  
\begin{subfigure}{0.24\linewidth}    
\includegraphics[width=\textwidth]{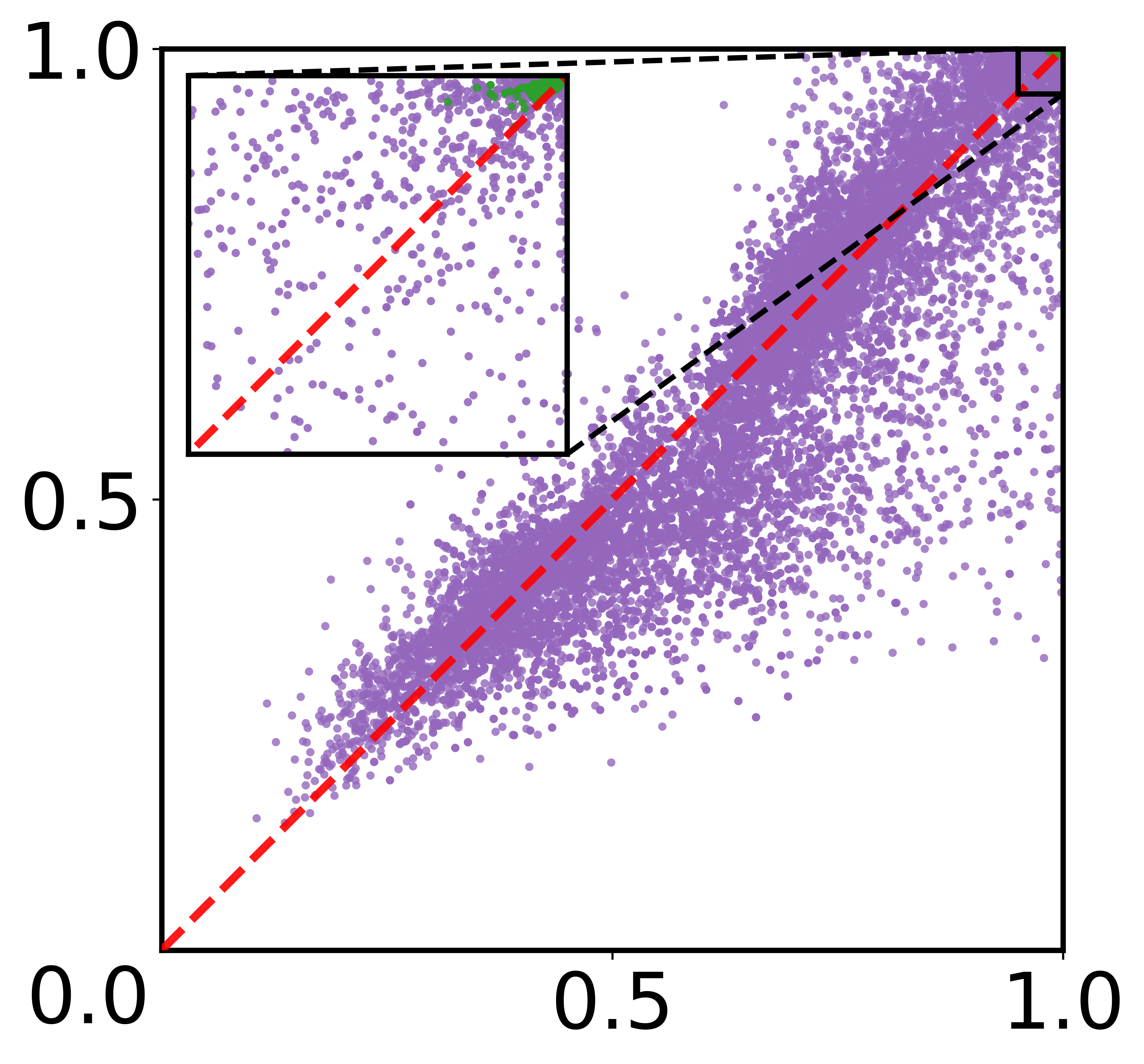} 
\caption{GSAT}  
\end{subfigure}  
\caption{Self-inconsistency patterns of four representative SI-GNNs on the BA-2MOTIFS dataset. Each plot compares first-pass (x-axis) and second-pass (y-axis) edge scores, with purple and green denoting unimportant and important edges (by ground-truth explanations), respectively. Deviations from the diagonal (red dashed line) mainly occur on a subset of unimportant edges.}
\label{fig:self-inconsistency}
\vspace{-0.2cm}
\end{figure}

Explanation self-consistency characterizes whether an SI-GNN can reproduce its own explanation under re-explanation. 
Given an input graph $G$, the SI-GNN first produces an explanation $G_s^{(1)} = h_{G_s}(G)$ and the corresponding explanatory graph $G_s^{(1)} = G \odot M^{(1)}$.
The model is then reapplied to $G_s^{(1)}$, yielding a second explanation $G_s^{(2)} = h_{G_s}(G_s^{(1)})= G \odot M^{(2)}$.
Explanation self-consistency is measured by the discrepancy between the two explanations:
\begin{align}
    \mathrm{ESC}(G_s^{(1)}, G_s^{(2)}) = d(G_s^{(1)}, G_s^{(2)}) = d(M^{(1)}, M^{(2)}),
\end{align}
where $d(\cdot,\cdot)$ denotes a discrepancy measure (e.g., mean absolute difference) between edge masks.

Recent work~\citep{tai2026self} showed that self-consistency is closely related to faithfulness~\citep{yuan2021explainability}: if an explanatory graph subset truly preserves the model's decision-relevant evidence, then reapplying the model to this graph subset should yield a consistent explanation. However, existing SI-GNNs were found to be self-inconsistent. On datasets with human-annotated ground-truth explanations, self-inconsistency mainly appears on unimportant edges, while important edges remain relatively stable across explanatory passes, as shown in \cref{fig:self-inconsistency}. Building on these observations, Tai et al. \cite{tai2026self} investigated whether training SI-GNNs toward self-consistency can improve explanation quality. Nevertheless, the mechanism behind self-inconsistency remains unclear. 
\section{Why Self-Inconsistency Arises in SI-GNNs}

In this section, we first show that re-explanation perturbs local message-passing contexts, and then use a latent signal hypothesis to explain why only some edges are sensitive to this perturbation. All proofs of theoretical results appearing from this section onward are deferred to \cref{app:proof}.

\subsection{Re-Explanation Perturbs Local Contexts}\label{subsec:local-context}

We begin by analyzing what changes during re-explanation. Formally, consider an $L$-layer GNN encoder $h_Z$. At the $l$-th layer, the representation of node $v_i$ is updated as
\begin{align}
    \mathbf{h}_i^{(l)}
    =
    \phi^{(l)}
    \Big(
        \mathbf{h}_i^{(l-1)},
        \{
        \mathbf{h}_j^{(l-1)}
        :
        v_j \in \mathcal{N}(v_i)
        \}
    \Big),
\end{align}
where $\mathcal{N}(v_i)$ denotes the neighbors of $v_i$, and $\phi^{(l)}$ represents the message aggregation function. 

During re-explanation, the explanatory graph $G_s^{(1)} = G \odot M^{(1)}$ perturbs the original graph through edge-wise masking. 
Although the underlying edge set may remain unchanged under soft masking, the message-passing context is altered because different nodes contribute with different weights. 
Since edge importance scores are computed from node representations produced by message passing (see \cref{app:setting} for implementation details), re-explanation naturally perturbs these scores.

The above analysis explains why self-inconsistency can arise even when the SI-GNN parameters remain unchanged.
However, it does not explain the heterogeneous behavior observed in \cref{fig:self-inconsistency}: only some unimportant edges exhibit score variations, whereas important edges remain stable across explanatory passes. 
We next address this question from a latent signal perspective.

\subsection{A Latent Signal Perspective}\label{subsec:latent-signal}

The stability of certain edges under re-explanation suggests that their scores are not determined by local context. Since re-explanation changes the surrounding message-passing context, the signals that determine their scores should arise from context-invariant components of the edge-scoring representation, such as intrinsic features derived from its incident nodes rather than information aggregated from changing neighborhoods. Edges associated with such signals are expected to remain stable across explanatory passes, whereas edges lacking these signals are more likely to exhibit substantial score variations. We formalize this intuition as follows.

\begin{hypothesis}[Latent edge signal assignment]
\label{assump:three_state}
Let $\mathcal{E}$ be the edge set of an input graph. 
We characterize the edge-scoring behavior of an SI-GNN through three latent states. 
The edge set is partitioned into three disjoint subsets
$\mathcal{E}_p$, $\mathcal{E}_n$, and $\mathcal{E}_c$, corresponding to positive-signal edges, negative-signal edges, and no-signal (context-driven) edges, respectively:
\begin{align}
\mathcal{E}=\mathcal{E}_p\cup\mathcal{E}_n\cup\mathcal{E}_c.
\end{align}
Edges in $\mathcal{E}_p$ receive high scores in $[x_+,1]$, edges in $\mathcal{E}_n$ receive low scores in $[0,x_-]$, and edges in $\mathcal{E}_c$ receive scores in $[0,1]$ determined primarily by their local message-passing context.
\end{hypothesis}

Prior evidence suggests that SI-GNNs typically assign high scores to important edges~\cite{tai2025redundancy,tai2026self}. Therefore, we expect ground-truth important edges to belong to $\mathcal{E}_p$, while unimportant edges may fall into any of the three states. When the explanatory graph subset is fed back into the same SI-GNN, the local message-passing context is perturbed. Under Hypothesis~\ref{assump:three_state}, since edges in $\mathcal{E}_p$ and $\mathcal{E}_n$ are supported by latent signals, their importance scores remain stable across explanatory passes. In contrast, edges in $\mathcal{E}_c$ are context-driven, and their scores may change as the structure changes.

\subsection{Conciseness Regularization Shapes Signal Allocation}\label{subsec:budget-signal}

The latent signal perspective above explains why only some edges exhibit self-inconsistency. A remaining question is what determines whether an edge becomes signal-driven or context-driven. Here, we analyze this question from the perspective of conciseness regularization, which constrains the total importance mass that an SI-GNN can assign to edges. We abstract this effect as an effective explanation budget $K$\footnote{This abstraction is natural and recognized by literature~\cite{zhang2022protgnn,tai2025redundancy} because most explanation methods implicitly impose a budget on explanations through constraints on size, sparsity, or related quantities.} and study how this budget constrains latent signal assignment.

Let $\mu_p$, $\mu_n$, and $\mu_c$ denote the average scores of edges in $\mathcal{E}_p$, $\mathcal{E}_n$, and $\mathcal{E}_c$, respectively. 
By Hypothesis~\ref{assump:three_state},
\begin{align}
    \mu_p\in[x_+,1], \qquad
    \mu_n\in[0,x_-], \qquad
    \mu_c\in[0,1],
\end{align}
where positive-signal edges contribute large importance mass, negative-signal edges contribute little importance mass, and context-driven edges contribute uncertain importance mass due to their sensitivity to local message-passing contexts. Consequently, under a fixed explanation budget, the proportions of positive-signal, negative-signal, and context-driven edges cannot vary independently. The following proposition formalizes this constraint.

\begin{proposition}[Budget-constrained signal allocation]
\label{prop:budget_feasibility}
Under Hypothesis~\ref{assump:three_state}, let $q\in(0,1)$ denote the target probability that the explanation budget is satisfied, and define:
\begin{align}
    c_q = \mu_c + \frac{1}{2}\sqrt{\frac{q}{1-q}} .
\end{align}
A sufficient condition for the explanation budget to be satisfied with probability at least $q$ is:
\begin{align}
K
\ge
|\mathcal{E}_p|\mu_p
+
|\mathcal{E}_n|\mu_n
+
|\mathcal{E}_c|c_q .
\label{eq:budget_basic}
\end{align}
Equivalently, since $\mathcal{E}_p$, $\mathcal{E}_n$, and $\mathcal{E}_c$ form a partition of $\mathcal{E}$, this condition can be written as:
\begin{align}
K
\ge
|\mathcal{E}|c_q
+
|\mathcal{E}_p|(\mu_p-c_q)
-
|\mathcal{E}_n|(c_q-\mu_n).
\label{eq:budget_tradeoff}
\end{align}
\end{proposition}

The proposition provides a budget-based perspective on how conciseness regularization shapes latent signal allocation. When the explanation budget is sufficient, important edges are expected to receive stable positive signals in order to support correct predictions. Since the budget remains sufficient, some unimportant edges may also retain positive signals, while many others remain context-driven. As the explanation budget gradually decreases, the model must reduce the total importance mass assigned to edges. Under Eq.~\eqref{eq:budget_tradeoff}, this increasingly favors reallocating unimportant edges from the positive-signal state toward the context-driven or negative-signal states, while some context-driven edges may further transition into the negative-signal state. During this stage, truly important edges can still preserve stable positive signals due to the supervision imposed by the downstream prediction objective. However, when the explanation budget becomes excessively tight, maintaining downstream prediction performance and satisfying the budget constraint become increasingly incompatible. In this regime, even some important edges may lose positive signals and become more context-sensitive, causing substantial degradation in downstream classification performance.

\textbf{Empirical support.} To provide empirical support for our latent signal hypothesis and signal allocation perspective, we analyze the relationship between edge score variation and surrounding context variation during re-explanation. For each edge $e_{ij}$, we define its score variation as
\begin{align}
    \Delta s_{ij}
    =
    \left|
    m_{ij}^{(1)}
    -
    m_{ij}^{(2)}
    \right|,
\end{align}
where $m_{ij}^{(1)} = M_{ij}^{(1)}$ and $m_{ij}^{(2)} = M_{ij}^{(2)}$ denote the importance scores of edge $e_{ij}$ in the first- and second-pass masks $M^{(1)}$ and $M^{(2)}$, respectively. We further quantify the variation of its surrounding message-passing context by averaging the score variations of neighboring edges:
\begin{align}
    \Delta c_{ij}
    =
    \frac{1}{|\mathcal{N}(e_{ij})|}
    \sum_{e_{uv}\in\mathcal{N}(e_{ij})}
    \left|
    m_{uv}^{(1)}
    -
    m_{uv}^{(2)}
    \right|,
\end{align}
where $\mathcal{N}(e_{ij})$ denotes the set of neighboring edges of $e_{ij}$, excluding the edge itself. Intuitively, $\Delta c_{ij}$ measures how strongly the local message-passing context surrounding $e_{ij}$ changes during re-explanation. We then compute Pearson, Spearman, and Kendall correlations between $\Delta s_{ij}$ and $\Delta c_{ij}$. Pearson correlation measures linear association, while Spearman and Kendall correlations measure rank-based monotonic association, with Kendall being more conservative and less sensitive to extreme ranks. If edges are primarily context-driven, edges experiencing larger context perturbations should also exhibit larger score variations. \cref{fig:context-correlation} reports these correlations under different conciseness regularization strengths, separately for important and unimportant edges.

As shown in \cref{fig:context-correlation}, the correlation between edge score variation and context variation differs significantly between important and unimportant edges and evolves with the regularization strength $\beta$. When $\beta$ is small, unimportant edges exhibit relatively higher correlation, indicating stronger context dependence, while important edges remain largely stable. As $\beta$ increases, the correlation of important edges gradually rises, suggesting that even important edges become more sensitive to contextual perturbations under stronger regularization. When the regularization becomes strong ($\beta=5\times10^{-1}$), the correlation of important edges increases substantially, and at $\beta=1$, it can even exceed that of unimportant edges. At the same time, downstream classification performance degrades noticeably: the accuracy, which remains consistently high (around $95\%-97\%$) under weaker regularization, drops to approximately $93\%$ at $\beta=5\times10^{-1}$ and further declines to $81\%$ at $\beta=1$. This behavior aligns with our budget-constrained signal allocation perspective: once the explanation budget becomes overly restrictive, preserving prediction performance and satisfying the budget constraint become incompatible, causing even important edges to lose stable positive signals and become increasingly context-sensitive.

Note that the moderate correlations observed for unimportant edges provide empirical support for the existence of negative signals in Hypothesis~\ref{assump:three_state}. If unimportant edges were purely context-driven, their score variations would be expected to show stronger dependence on context variations. Instead, the observed moderate correlations are more consistent with a mixture of context-driven edges and negative-signal edges, where the latter dilute the overall context sensitivity of unimportant edges.

\begin{figure}[t]  
\centering     
\includegraphics[width=\textwidth]{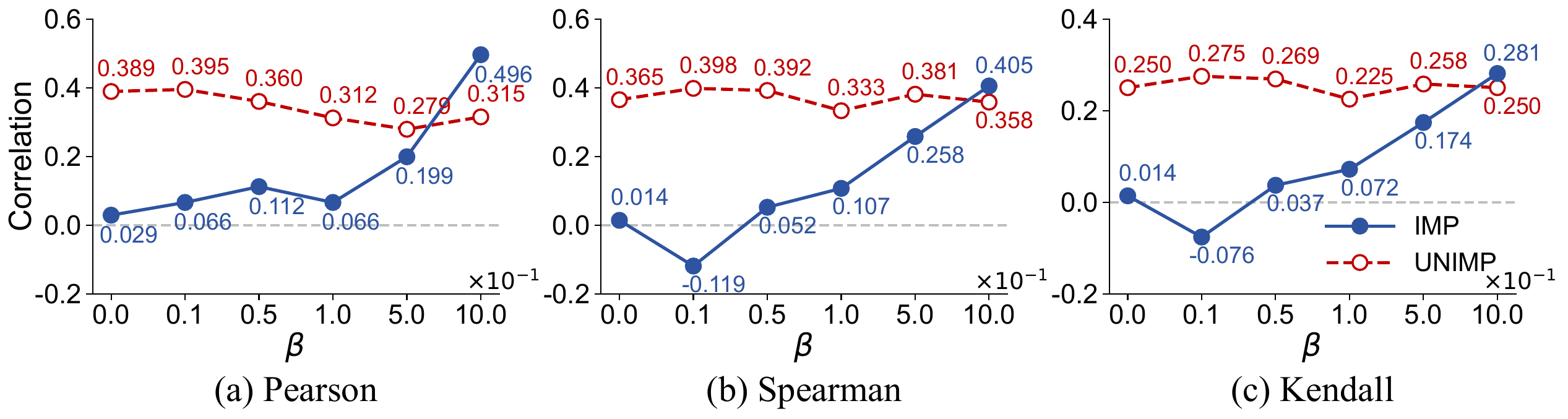}   
\vspace{-0.5cm}
\caption{Correlation analysis between edge score variation and local context variation on SMGNN over the BA-2MOTIFS dataset. Pearson, Spearman, and Kendall correlations are reported separately for important and unimportant edges under different conciseness regularization strengths.}
\label{fig:context-correlation}
\vspace{-0.2cm}
\end{figure}

\section{From Self-Inconsistency to Self-Denoising}

Existing literature defines an explanation as the set of high-scoring edges~\cite{ji2025comprehensive}. However, our analysis above suggests that some high-scoring edges are merely contextually amplified: their importance arises from surrounding message-passing contexts and is not reproduced under re-explanation. As a result, explanations inevitably mix positively signaled edges, which the model genuinely regards as important, with contextually amplified ones. Such contamination not only reduces explanation plausibility, but may also lead to serious consequences in high-stakes settings where explanations guide human decisions (e.g., scientific discovery or costly decision-making~\cite{wong2024discovery}).

\subsection{Self-Denoising: A Training-Free Explanation Calibration Strategy}\label{subsec:sd}

The self-inconsistency phenomenon provides an intrinsic signal for calibrating explanations. Intuitively, edges whose importance remains stable under re-explanation are more likely to be supported by latent signals. Motivated by this observation, we propose \textit{Self-Denoising (SD)}, a model-agnostic, simple, and efficient post-processing strategy to improve explanation quality. 

Specifically, for each edge $e_{ij}$, we measure its instability as
\begin{align}
    \Delta s_{ij} = \left| m_{ij}^{(1)} - m_{ij}^{(2)} \right|,
\end{align}
where $m_{ij}^{(1)}$ and $m_{ij}^{(2)}$ are the edge importance scores from the first and second explanatory passes, respectively. We then calibrate the edge scores by
\begin{align}\label{eq:sd}
    \tilde{m}_{ij}
    =
    \max\left(0, (1 - \eta\Delta s_{ij}) m_{ij}^{(1)}\right),
\end{align}
where $\eta \geq 0$ controls the denoising strength. In this way, SD shifts explanations from being driven by instantaneous scores to being supported by stable evidence.

\subsection{Self-Denoising vs. Out-of-Distribution}\label{subsec:ood}

SD improves explanation quality by suppressing context-driven edges. However, because SD directly modifies edge importance scores, it may introduce distribution shift. Below, we analyze the effect of the denoising strength $\eta$ from two perspectives: ranking correction and prediction stability.

\begin{proposition}[Ranking correction by SD]
\label{prop:sd_ranking}
Consider an important edge $e^+$ and an unimportant edge $e^-$. 
Let $m^+$ and $m^-$ denote their first-pass importance scores, and let $\Delta s^+$ and $\Delta s^-$ denote their instability scores under re-explanation. 
Suppose the original explanation ranks them incorrectly, i.e., $m^+ < m^-$, and the unimportant edge is more unstable, i.e.,
\begin{align}
    \Delta s^- > \Delta s^+ .
\end{align}
SD corrects their relative order, i.e., $\tilde{m}^+ > \tilde{m}^-$ whenever
\begin{align}
    \eta >
    \frac{m^- - m^+}
    {m^- \Delta s^- - m^+ \Delta s^+}.
\end{align}
\end{proposition}
\begin{proposition}[Prediction stability under SD] \label{prop:sd_stability}
Let $G_s^{(1)} = G \odot M^{(1)}$ and $\tilde{G}_s = G \odot \tilde{M}$. 
Assume the prediction function $f$ is locally differentiable with respect to the edge-weight vector. 
If we require the prediction shift induced by SD to be bounded by $\varepsilon>0$, i.e.,
\begin{align}
    |f(\tilde{G}_s)-f(G_s^{(1)})| \le \varepsilon,
\end{align}
then a sufficient condition is
\begin{align}
    \eta
    \le
    \frac{\varepsilon}
    {
    \|\nabla f(M^{(1)})\|_\infty
    \sum_{e_{ij}\in\mathcal{E}} m_{ij}^{(1)}\Delta s_{ij}
    }.
\end{align}
\end{proposition}
Proposition~\ref{prop:sd_ranking} shows that $\eta$ should be sufficiently large to correct context-driven ranking errors, while Proposition~\ref{prop:sd_stability} shows that $\eta$ should not be overly large to avoid excessive prediction shift. Together, they reveal a trade-off between explanation calibration and prediction stability. A complementary stability bound under stochastic masking is provided in \cref{prop:sd_stochastic_stability} in \cref{app:proof}.

\subsection{Practical Consideration}\label{subsec:eta-select}

The above trade-off raises a practical question: \textit{how should $\eta$ be selected? }

A common way to select hyperparameters is to use validation accuracy: if an explanation is informative, the model should still make correct predictions based on it. However, due to the out-of-distribution (OOD) issue introduced by SD, this strategy becomes unreliable without adaptation. Therefore, it is necessary to first mitigate the impact of distribution shift on the model before using validation accuracy for selection. Recall that an SI-GNN predicts by first generating an explanatory graph, then encoding it with a GNN encoder, and finally feeding the graph representation into a classifier. Since the encoder is trained with stochastically sampled subgraphs, it is generally considered to be relatively robust to moderate structural perturbations~\cite{miao2022interpretable,luo2024towards,tai2025redundancy}. 

In practice, we freeze the explainer and encoder, and fine-tune only the classifier on SD-calibrated explanations for a few epochs (10 epochs in our experiments). After this lightweight adaptation, validation accuracy can better reflect whether SD preserves task-relevant information. We select $\eta$ based on the adapted validation accuracy and empirically examine this strategy in the next section.
\section{Experiments}

\textbf{Baselines.} We integrate SD into four representative SI-GNNs: GAT~\cite{velivckovic2018graph}, CAL~\cite{sui2022causal}, SMGNN~\cite{azzolin2025beyond}, and GSAT~\cite{miao2022interpretable}, which respectively represent the \textit{attention-based}, \textit{causal-based}, \textit{size-constrained}, and \textit{MI–based} design paradigms of SI-GNNs. We also compare with EE~\citep{tai2025redundancy}, a post-hoc explanation calibration method, and study whether SD provides complementary gains.

\textbf{Metrics.} We adopt four complementary metrics to comprehensively evaluate explanation quality. ROC-AUC (AUC), computed against the ground-truth explanations, measures the \textit{accuracy} (a.k.a \textit{plausibility}) of the generated explanations. Fidelity$^-$ (FID) assesses the \textit{faithfulness} of the explanation, indicating whether it truly corresponds to the model's internal decision process. Accuracy (ACC) evaluates the downstream classification performance based on the explanatory subset, reflecting its \textit{informativeness}. Sparsity (SPA) quantifies the \textit{conciseness} of the explanation. 

\textbf{Datasets.} Consistent with previous studies~\cite{tai2025redundancy,tai2026self}, we evaluate SD on four widely used benchmarks: one synthetic dataset, BA-2MOTIFS~\cite{luo2020parameterized}, and three real-world datasets, 3MR~\cite{rao2022quantitative}, BENZENE~\cite{morris2020tudataset}, and MUTAGENICITY~\cite{morris2020tudataset}. Details of the experimental setup are provided in \cref{app:setting}.



\subsection{Overall Performance of SD}

\begin{table*}[t]      
\caption{Experimental results of SD on SI-GNNs with GIN backbone across four datasets. \textbf{Bold} indicates better performance, while \underline{underlined} results are statistically significant ($p < 0.05$).}
\vspace{-0.2cm}
\label{tab:main_results}    
\begin{center}    
\setlength{\tabcolsep}{2pt} 
\resizebox{\linewidth}{!}{    
\begin{tabular}{ccccccccc}
\toprule        
\multirow{2}{*}{\vspace{-2mm} Method}
& \multicolumn{2}{c}{BA-2MOTIFS}         
& \multicolumn{2}{c}{3MR}         
& \multicolumn{2}{c}{BENZENE}         
& \multicolumn{2}{c}{MUTAGENICITY} \\        
\cmidrule(lr){2-3}\cmidrule(lr){4-5}\cmidrule(lr){6-7}\cmidrule(lr){8-9}        
& $\uparrow$ AUC (\%) & $\downarrow$ FID (\%)
& $\uparrow$ AUC (\%) & $\downarrow$ FID (\%)         
& $\uparrow$ AUC (\%) & $\downarrow$ FID (\%)         
& $\uparrow$ AUC (\%) & $\downarrow$ FID (\%) \\        
\midrule        
GAT~\citep{velivckovic2018graph}        
& 99.31$\pm$0.35 & 2.50$\pm$7.50
& 97.25$\pm$0.67 & {2.63$\pm$1.93}
& 83.51$\pm$2.24 & {2.06$\pm$0.65}
& 91.71$\pm$6.00 & 0.37$\pm$0.53 \\                
GAT+SD
& \textbf{99.47$\pm$0.28} & \underline{\textbf{0.80$\pm$2.40}}
& \underline{\textbf{98.68$\pm$0.40}} & {\textbf{2.15$\pm$1.17}}
& \underline{\textbf{87.39$\pm$1.57}} & \underline{\textbf{1.24$\pm$0.73}}
& \underline{\textbf{93.14$\pm$5.29}} & 1.22$\pm$1.14 \\ 
GAT+SD*
& \textbf{99.47$\pm$0.28} & \underline{\textbf{0.00$\pm$0.00}}
& \underline{\textbf{98.68$\pm$0.40}} & 6.37$\pm$4.41
& \underline{\textbf{87.39$\pm$1.57}} & \underline{\textbf{1.56$\pm$0.93}}
& \underline{\textbf{93.14$\pm$5.29}} & 1.25$\pm$1.34 \\
\midrule        
CAL~\citep{sui2022causal}               
& 98.64$\pm$1.33 & 17.10$\pm$17.21
& 96.25$\pm$1.59 & 8.17$\pm$2.52
& 77.64$\pm$3.05 & 4.97$\pm$3.44
& 96.39$\pm$1.39 & 2.06$\pm$1.14 \\  

CAL+SD
& \textbf{98.71$\pm$1.25} & 17.60$\pm$17.04
& \underline{\textbf{97.74$\pm$1.13}} & 8.24$\pm$2.54
& \textbf{79.12$\pm$3.05} & {4.90$\pm$3.52}
& \textbf{96.88$\pm$1.21} & 2.47$\pm$1.35 \\
CAL+SD*
& \textbf{98.71$\pm$1.25} & \underline{\textbf{1.50$\pm$2.38}}
& \underline{\textbf{97.74$\pm$1.13}} & \underline{\textbf{2.78$\pm$0.77}}
& \textbf{79.12$\pm$3.05} & \underline{\textbf{3.79$\pm$1.84}}
& \textbf{96.88$\pm$1.21} & 2.36$\pm$1.60 \\
\midrule        
SMGNN~\citep{azzolin2025beyond}              
& 99.32$\pm$0.36 & 0.50$\pm$0.92
& 97.00$\pm$0.77 & 1.90$\pm$1.01
& 84.38$\pm$2.71 & 2.07$\pm$0.62
& 98.12$\pm$0.39 & 1.72$\pm$1.09 \\      

SMGNN+SD
& \textbf{99.47$\pm$0.24} & \textbf{0.30$\pm$0.64}
& \underline{\textbf{98.59$\pm$0.51}} & \underline{\textbf{0.80$\pm$0.51}}
& \underline{\textbf{88.05$\pm$1.64}} & \underline{\textbf{1.12$\pm$0.50}}
& \underline{\textbf{98.51$\pm$0.36}} & \underline{\textbf{1.28$\pm$0.93}} \\     
SMGNN+SD*
& \textbf{99.47$\pm$0.24} & \underline{\textbf{0.00$\pm$0.00}}
& \underline{\textbf{98.59$\pm$0.51}} & 4.33$\pm$5.82
& \underline{\textbf{88.05$\pm$1.64}} & \underline{\textbf{1.08$\pm$0.58}}
& \underline{\textbf{98.51$\pm$0.36}} & \textbf{1.59$\pm$0.62} \\
\midrule        
GSAT~\citep{miao2022interpretable}      
& {99.30$\pm$0.47} & 0.00$\pm$0.00 
& 98.38$\pm$0.31 & 0.90$\pm$0.28
& 90.66$\pm$0.89 & 1.80$\pm$0.84 
& 99.01$\pm$0.31 & 1.11$\pm$0.61 \\          

GSAT+SD
& \textbf{99.44$\pm$0.40} & 0.00$\pm$0.00
& \underline{\textbf{99.22$\pm$0.23}} & \textbf{0.62$\pm$0.40}
& \underline{\textbf{92.17$\pm$0.71}} & \underline{\textbf{1.13$\pm$0.49}}
& \textbf{99.11$\pm$0.24} & \textbf{1.11$\pm$0.40} \\
GSAT+SD*
& \textbf{99.44$\pm$0.40} & 0.00$\pm$0.00
& \underline{\textbf{99.22$\pm$0.23}} & \underline{\textbf{0.38$\pm$0.59}}
& \underline{\textbf{92.17$\pm$0.71}} & \underline{\textbf{0.89$\pm$0.27}}
& \textbf{99.11$\pm$0.24} & 1.22$\pm$0.66 \\
\\[-3mm] \hdashline \\[-3mm]            
\diagbox[width=2.5cm]{}{}        
& $\uparrow$ ACC (\%) & $\downarrow$ SPA (\%)      
& $\uparrow$ ACC (\%) & $\downarrow$ SPA (\%)         
& $\uparrow$ ACC (\%) & $\downarrow$ SPA (\%)         
& $\uparrow$ ACC (\%) & $\downarrow$ SPA (\%) \\ 
\midrule 
GAT~\citep{velivckovic2018graph}        
& 97.10$\pm$8.70 & 69.98$\pm$6.02  
& 96.54$\pm$1.46 & {26.31$\pm$2.25}   
& 91.60$\pm$0.66 & 64.16$\pm$4.14   
& 92.97$\pm$0.78 & {88.66$\pm$8.80} \\      

GAT+SD                                  
& \textbf{98.60$\pm$4.20} & \underline{\textbf{64.35$\pm$6.49}}  
& \textbf{97.20$\pm$1.38} & \underline{\textbf{20.13$\pm$1.75}}           
& 91.39$\pm$1.23 & \underline{\textbf{53.01$\pm$5.18}}   
& \textbf{92.94$\pm$1.02} & \underline{\textbf{74.17$\pm$21.00}} \\
GAT+SD*
& \underline{\textbf{100.00$\pm$0.00}} & \underline{\textbf{64.35$\pm$6.49}}  
& \underline{\textbf{99.07$\pm$0.41}} & \underline{\textbf{20.13$\pm$1.75}} 
& \underline{\textbf{92.66$\pm$0.97}} & \underline{\textbf{53.01$\pm$5.18}}   
& \textbf{93.28$\pm$0.63} & \underline{\textbf{74.17$\pm$21.00}} \\
\midrule        
CAL~\citep{sui2022causal}               
& 91.90$\pm$11.68 & {66.99$\pm$6.50}   
& 94.22$\pm$2.11 & {26.65$\pm$1.96}  
& 84.31$\pm$5.66 & {60.80$\pm$10.75}   
& 91.22$\pm$1.37 & 73.09$\pm$14.58    \\

CAL+SD                                  
& \underline{\textbf{92.20$\pm$11.14}} & \underline{\textbf{66.09$\pm$6.72}}  
& 93.91$\pm$2.23 & \underline{\textbf{25.30$\pm$1.78}}   
& 84.28$\pm$5.81 & \underline{\textbf{60.06$\pm$10.86}}   
& 91.15$\pm$1.34 & \underline{\textbf{67.06$\pm$17.08}}  \\
CAL+SD*
& \underline{\textbf{99.80$\pm$0.40}} & \underline{\textbf{66.09$\pm$6.72}} 
& \underline{\textbf{97.58$\pm$1.89}} & \underline{\textbf{25.30$\pm$1.78}}
& \underline{\textbf{90.04$\pm$1.12}} & \underline{\textbf{60.06$\pm$10.86}} 
& \underline{\textbf{92.84$\pm$0.88}} & \underline{\textbf{67.06$\pm$17.08}} \\
\midrule        
SMGNN~\citep{azzolin2025beyond}              
& 95.50$\pm$12.52 & {61.06$\pm$4.47}  
& 97.30$\pm$1.31 & {19.60$\pm$2.13}   
& 91.12$\pm$0.62 & {46.73$\pm$6.39}   
& 89.53$\pm$0.99 & 20.50$\pm$4.78    \\      

SMGNN+SD                                
& 95.50$\pm$13.50 & \underline{\textbf{54.57$\pm$4.62}}  
& \textbf{98.44$\pm$1.06} & \underline{\textbf{11.44$\pm$1.40}}  
& \textbf{91.34$\pm$1.14} & \underline{\textbf{38.46$\pm$5.41}}   
& 89.19$\pm$0.79 & \underline{\textbf{16.57$\pm$3.61}}  \\  
SMGNN+SD*
& \underline{\textbf{100.00$\pm$0.00}} & \underline{\textbf{54.57$\pm$4.62}}
& \underline{\textbf{99.03$\pm$0.57}} & \underline{\textbf{11.44$\pm$1.40}}  
& \underline{\textbf{92.07$\pm$0.79}} & \underline{\textbf{38.46$\pm$5.41}}  
& \underline{\textbf{91.18$\pm$1.24}} & \underline{\textbf{16.57$\pm$3.61}}  \\  
\midrule        
GSAT~\citep{miao2022interpretable}      
& 100.00$\pm$0.00 & 74.44$\pm$2.63  
& 98.55$\pm$0.80 & 63.88$\pm$3.92   
& 91.48$\pm$0.87 & 57.10$\pm$10.18  
& 92.43$\pm$1.00 & 44.39$\pm$6.10    \\        

GSAT+SD                                 
& 100.00$\pm$0.00 & \underline{\textbf{70.24$\pm$2.92}}  
& \textbf{98.89$\pm$0.80} & \underline{\textbf{62.39$\pm$2.42}}  
& \textbf{91.72$\pm$0.92} & \underline{\textbf{54.44$\pm$8.14}}  
& 91.99$\pm$0.93 & \underline{\textbf{39.16$\pm$6.20}}   \\ 
GSAT+SD*
& 100.00$\pm$0.00 & \underline{\textbf{70.24$\pm$2.92}}  
& \underline{\textbf{99.55$\pm$0.27}} & \underline{\textbf{62.39$\pm$2.42}}  
& \underline{\textbf{92.91$\pm$0.47}} & \underline{\textbf{54.44$\pm$8.14}}  
& \underline{\textbf{93.38$\pm$0.34}} & \underline{\textbf{39.16$\pm$6.20}} \\
\bottomrule
\end{tabular}
}
\end{center}
\vspace{-0.2cm}
\end{table*}

\begin{figure}[t]
    \vspace{-0.3cm}
    \centering    
    \includegraphics[width=\linewidth]{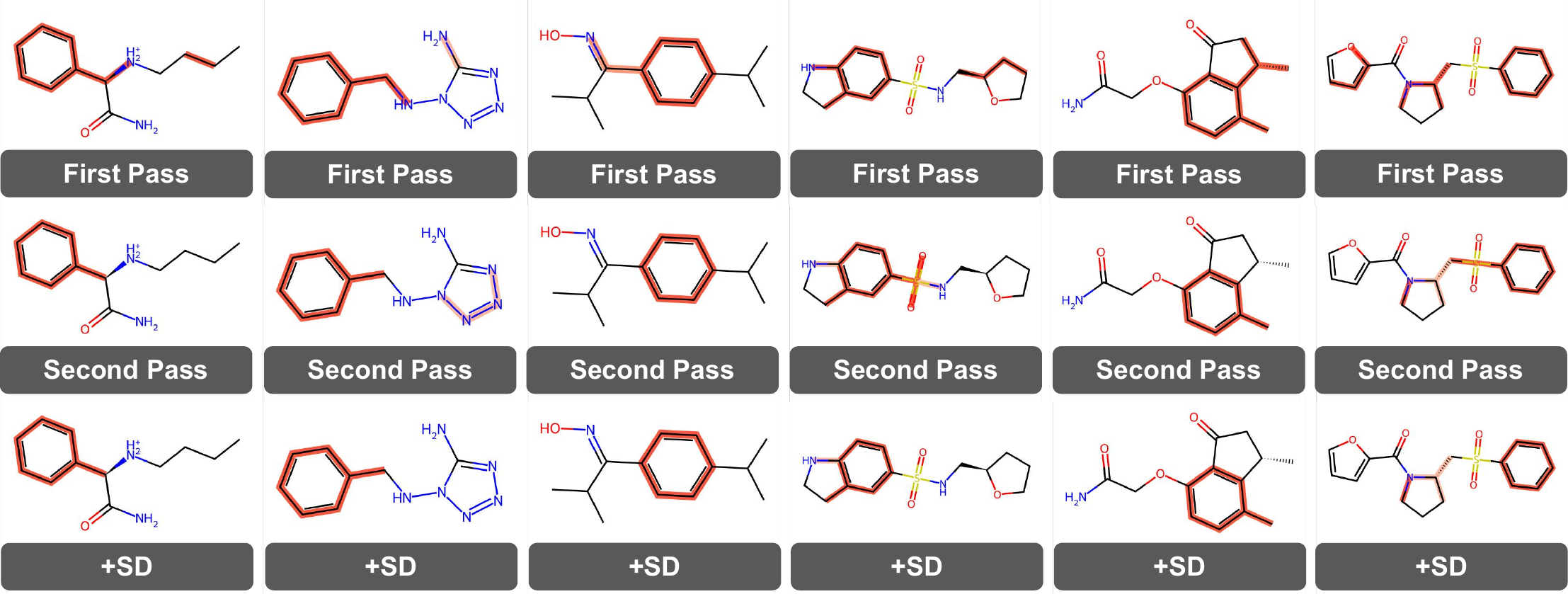}    
    \caption{Cases of SD on the BENZENE dataset. Top: first-pass explanation. Middle: second-pass explanation. Bottom: after SD. Unstable edges are suppressed while stable structures are preserved.}
    \label{fig:sd-case}
    \vspace{-0.2cm}
\end{figure}

We first evaluate the overall performance of SD on the GIN~\cite{xu2019how} backbone across four SI-GNNs and four datasets. Here, SD denotes the training-free explanation calibration step, while SD* denotes SD with the lightweight classifier adaptation described in \cref{subsec:eta-select}, used only to mitigate the induced distribution shift. As shown in \cref{tab:main_results}, SD consistently improves explanation quality. In particular, AUC is improved in all 16 cases, with gains of up to 3.9\%, while SPA is significantly reduced, indicating more concise and less redundant explanations. FID shows mixed behavior, which we believe is related to distribution shift and explanation redundancy, as further discussed in \cref{app:metrics}. In terms of predictive performance, applying SD to the original models leads to only minor accuracy changes, with most cases remaining stable and several even showing improvements. When further applying lightweight classifier adaptation (SD*), prediction accuracy improves consistently across all 16 cases, with gains of up to 8\%. We provide case studies in \cref{fig:sd-case} for intuitive understanding of how SD works, and report corresponding results on other GNN backbones in \cref{app:results}.



In \cref{subsec:eta-select}, we proposed selecting $\eta$ using validation accuracy after lightweight classifier adaptation. Here, we empirically validate this strategy. As shown in \cref{fig:eta-smgnn}, accuracy improves within a moderate range of $\eta$, suggesting that SD removes context-driven noise while preserving task-relevant information. When $\eta$ becomes too large, accuracy drops, indicating over-denoising. These results show that adapted validation accuracy is a practical proxy for choosing $\eta$ and can effectively capture the trade-off between explanation calibration and information loss.

\subsection{Complementarity Between SD and EE}\label{exp:sd-ee}

\begin{table*}[t]
\caption{Experimental results of combining SD with EE on SI-GNNs with GIN backbone.}
\vspace{-0.2cm}
\label{tab:with_EE}    
\begin{center}    
\setlength{\tabcolsep}{2pt}
\resizebox{\linewidth}{!}{    
\begin{tabular}{ccccccccc}        
\toprule        
\multirow{2}{*}{\vspace{-2mm} Method}         
& \multicolumn{2}{c}{BA-2MOTIFS}         
& \multicolumn{2}{c}{3MR}         
& \multicolumn{2}{c}{BENZENE}         
& \multicolumn{2}{c}{MUTAGENICITY} \\        
\cmidrule(lr){2-3}\cmidrule(lr){4-5}\cmidrule(lr){6-7}\cmidrule(lr){8-9}        
& $\uparrow$ AUC (\%) & $\uparrow$ ACC (\%)         
& $\uparrow$ AUC (\%) & $\uparrow$ ACC (\%)         
& $\uparrow$ AUC (\%) & $\uparrow$ ACC (\%)         
& $\uparrow$ AUC (\%) & $\uparrow$ ACC (\%) \\        
\midrule        

GAT+EE
& 99.56$\pm$0.16 & {100.00$\pm$0.00}
& 98.37$\pm$0.13 & 97.43$\pm$0.37
& 89.59$\pm$1.37 & {92.47$\pm$0.19}
& 97.07$\pm$0.98 & 93.81$\pm$0.29 \\

GAT+SD+EE
& \textbf{99.67$\pm$0.12} & 100.00$\pm$0.00 & \textbf{99.35$\pm$0.08} & \textbf{97.78$\pm$0.47} & \textbf{93.31$\pm$0.76} & 92.40$\pm$0.21 & \textbf{97.89$\pm$0.64} & \textbf{93.96$\pm$0.30} \\ 

\midrule        
CAL+EE
& {99.53$\pm$0.14} & 98.05$\pm$2.11
& 98.16$\pm$0.13 & {95.60$\pm$0.93}
& 84.83$\pm$1.23 & {87.59$\pm$1.82}
& 98.76$\pm$0.43 & 91.82$\pm$0.68 \\

CAL+SD+EE
& \textbf{99.55$\pm$0.12} & \textbf{98.19$\pm$2.11} & \textbf{99.10$\pm$0.09} & 95.40$\pm$0.97 & \textbf{85.90$\pm$1.22} & 87.52$\pm$1.88 & \textbf{98.88$\pm$0.35} & \textbf{91.87$\pm$0.67} \\ 

\midrule        
SMGNN+EE
& 99.59$\pm$0.16 & {100.00$\pm$0.00}
& 98.25$\pm$0.21 & 97.85$\pm$0.38
& 91.38$\pm$0.96 & 92.16$\pm$0.24
& 98.96$\pm$0.16 & {90.36$\pm$0.70} \\

SMGNN+SD+EE
& \textbf{99.66$\pm$0.10} & 100.00$\pm$0.00 & \textbf{99.25$\pm$0.05} & \textbf{99.40$\pm$0.20} & \textbf{93.99$\pm$0.64} & \textbf{92.31$\pm$0.40} & \textbf{99.09$\pm$0.11} & 90.05$\pm$0.67 \\

\midrule        
GSAT+EE
& 99.38$\pm$0.19 & 100.00$\pm$0.00
& 99.16$\pm$0.03 & 99.15$\pm$0.24
& 92.18$\pm$0.59 & 92.17$\pm$0.28
& 99.36$\pm$0.05 & {93.00$\pm$0.44} \\

GSAT+SD+EE
& \textbf{99.53$\pm$0.15} & 100.00$\pm$0.00 & \textbf{99.69$\pm$0.04} & \textbf{99.45$\pm$0.20} & \textbf{93.68$\pm$0.38} & \textbf{92.61$\pm$0.33} & \textbf{99.41$\pm$0.04} & 92.83$\pm$0.40 \\

\bottomrule
\end{tabular}
}
\end{center}
\vspace{-0.2cm}
\end{table*}

Besides SD, the other post-hoc explanation calibration strategy is Explanation Ensemble (EE)~\cite{tai2025redundancy}, which identifies unreliable edges based on cross-model inconsistency. Here, we investigate whether SD and EE can be combined. As shown in \cref{tab:with_EE}, applying SD before EE consistently yields better performance than using EE alone. This suggests that self-inconsistency captures a distinct source of explanation noise that is not fully addressed by cross-model inconsistency. In terms of efficiency, detecting self-inconsistency within a single model is significantly cheaper, introducing only about 4--6\% additional overhead, whereas EE increases computational cost by around 400\%. In practice, SD can serve as a lightweight calibration step, and when computational cost is not a concern, combining SD with EE can further improve the explanation quality.

\begin{figure}[t]    
    \centering    
    \includegraphics[width=\linewidth]{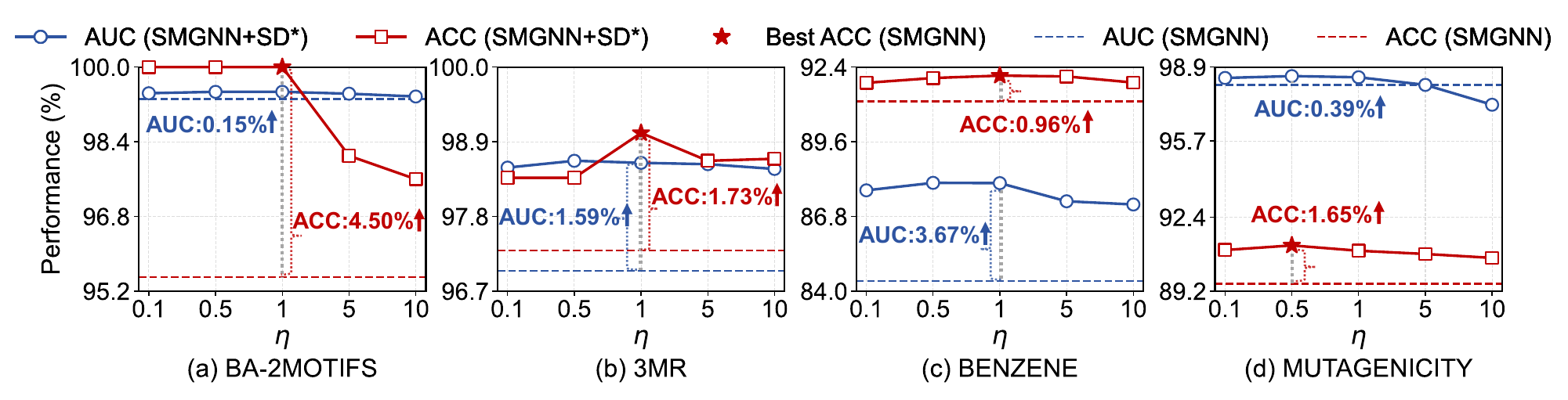}    
    \caption{Selecting $\eta$ via adapted predictive performance on SMGNN. $\eta$ is selected on validation data; test results are shown. Other SI-GNNs are reported in \cref{app:results}.}
    \label{fig:eta-smgnn}   
    \vspace{-0.2cm}
\end{figure}

\section{Conclusion}

In this work, we investigate an underexplored yet critical issue in SI-GNNs: self-inconsistency, where explanations may change under re-explanation. We revisit the message-passing mechanism of GNNs and identify perturbations in local contexts as the direct cause of this phenomenon. To explain why only certain edges are affected, we introduce a latent signal assignment perspective, which characterizes how models allocate importance to edges under conciseness regularization. Building on this understanding, we propose SD, a post-processing strategy that leverages self-inconsistency as an intrinsic signal to calibrate explanations. Experiments across multiple SI-GNN frameworks, GNN backbones, and benchmark datasets demonstrate that SD improves explanation quality. 


\bibliographystyle{unsrtnat}
\bibliography{main}

\appendix

\clearpage
\section{Extended Related Work}\label{app:related}

Unlike post-hoc methods that explain GNNs after training, SI-GNNs integrate explanation generation into the learning pipeline to ensure the model's rationale aligns with its prediction. Most of them follow a two-stage pipeline that identifies an explanatory graph subset via a learned mask and makes predictions based on it, with different methods varying in how the subset is extracted~\cite{tai2025redundancy}. Attention-based methods~\cite{velivckovic2018graph,knyazev2019understanding,lu2020gcan} utilize attention mechanisms~\cite{vaswani2017attention} to infer feature importance without additional conciseness regularization. Causal-learning methods~\cite{wu2022discovering,sui2022causal} leverage disentanglement or interventions~\cite{pearl2014interpretation} to uncover causal structures beyond spurious correlations. Size-constrained methods~\cite{luo2024towards,lin2020graph,azzolin2025beyond} explicitly regularize subset size or sparsity, promoting concise yet informative explanations. MI-based methods~\cite{yu2021graph,yu2022improving,miao2022interpretable,seo2023interpretable,chen2024interpretable} share a similar objective to size-constrained ones but achieve it by minimizing the MI between the explanatory subset and the full graph. Our work does not propose new SI-GNNs; instead, we focus on the trustworthiness of existing SI-GNNs. 

In addition to the above paradigms, there exist other variants of SI-GNNs, such as prototype-based methods. These methods make predictions by mapping input graphs to a set of learned semantic prototypes~\cite{dai2021towards,feng2022kergnns,zhang2022protgnn,xia2025graph}. However, they typically provide global, concept-level explanations rather than instance-level, edge-based interpretability, which is the focus of this work. Notably, certain instance-level variants of prototype-based methods (e.g., ProtGNN~\cite{zhang2022protgnn}) incorporate sparsity regularization and behave similarly to size-constrained SI-GNNs under a Lagrangian formulation. Therefore, the four SI-GNN paradigms considered in our main text can be viewed as representative of the dominant design space for instance-level explanations~\cite{tai2025redundancy}.

Beyond designing new SI-GNN architectures, a small but growing line of work has begun to investigate the trustworthiness of explanations under standard (i.e., no attack) usage conditions. Tai et al.~\cite{tai2025redundancy} identify a redundancy issue in SI-GNN explanations, showing that graph heterogeneity and validation-accuracy-based hyperparameter selection can make conciseness regularization too weak, leading to explanations contaminated by redundant and task-irrelevant edges. Azzolin et al.~\cite{azzolin2026gnn} show that overly strong sparsity regularization can lead to degenerate and unfaithful explanations. More recently, Tai et al. \cite{tai2026self} reveal the phenomenon of self-inconsistency, where explanations may change under re-explanation, and empirically connect it to explanation redundancy.

Among these works, two are particularly related to ours from different perspectives. First, Tai et al.~\cite{tai2025redundancy} propose Explanation Ensemble (EE), a post-hoc explanation calibration method. 
EE and SD share the same post-hoc calibration paradigm, but they identify unreliable edges through different signals: EE relies on cross-model inconsistency, whereas SD uses single-model self-inconsistency. Our experiments in the main text (\cref{exp:sd-ee}) show that SD is substantially more efficient than EE: SD requires only one additional forward pass (4\%--6\% additional overhead), whereas EE relies on multiple independently trained models (400\% additional overhead for 5 models). Meanwhile, the two strategies are complementary, suggesting that they capture different subsets of unreliable edges. 

Second, Tai et al.~\cite{tai2026self} investigate whether faithfulness-based objectives can be directly optimized during SI-GNN training. In this process, they identify a close connection between explanation self-consistency and faithfulness, and propose to improve explanation faithfulness by enforcing self-consistency during training. Their work mainly focuses on the effect of such optimization, while the mechanism behind why self-inconsistency arises remains underexplored. In contrast, our work studies why self-inconsistency arises and how it can be exploited as a post-hoc calibration signal. Moreover, their strategy requires fine-tuning SI-GNNs with an additional self-consistency objective and may depend on specific regularization settings, whereas SD is training-free, broadly applicable, and can be directly applied to pretrained SI-GNNs. We provide a detailed empirical comparison and discussion with training-time self-consistency regularization in \cref{app:with-sc-iclr}.


\section{Proofs of Theoretical Results}\label{app:proof}

\textbf{Proposition 1} (Budget-constrained signal allocation)\textbf{.}
\textit{Under Hypothesis~\ref{assump:three_state}, let $q\in(0,1)$ denote the target probability that the explanation budget is satisfied, and define:}
\begin{align}
    c_q = \mu_c + \frac{1}{2}\sqrt{\frac{q}{1-q}} .
\end{align}
\textit{A sufficient condition for the explanation budget to be satisfied with probability at least $q$ is:}
\begin{align}
K
\ge
|\mathcal{E}_p|\mu_p
+
|\mathcal{E}_n|\mu_n
+
|\mathcal{E}_c|c_q .
\end{align}
\textit{Equivalently, since $\mathcal{E}_p$, $\mathcal{E}_n$, and $\mathcal{E}_c$ form a partition of $\mathcal{E}$, this condition can be written as:}
\begin{align}
K
\ge
|\mathcal{E}|c_q
+
|\mathcal{E}_p|(\mu_p-c_q)
-
|\mathcal{E}_n|(c_q-\mu_n).
\end{align}

\begin{proof}
The deterministic contribution of positive-signal and negative-signal edges is:
\begin{align}
    |\mathcal{E}_p|\mu_p + |\mathcal{E}_n|\mu_n .
\end{align}
For each context-driven edge $e\in\mathcal{E}_c$, let $m_e\in[0,1]$ denote its context-dependent mask score, and define:
\begin{align}
    S_c = \sum_{e_{ij}\in\mathcal{E}_c} m_{ij} .
\end{align}
By definition,
\begin{align}
    \mathbb{E}[S_c] = |\mathcal{E}_c|\mu_c,
    \qquad
    0\le S_c\le |\mathcal{E}_c|.
\end{align}

We first bound the variance of $S_c$. 
Popoviciu's variance inequality states that, for any random variable $X$ supported on an interval $[a,b]$,
\begin{align}
    \mathrm{Var}(X) \le \frac{(b-a)^2}{4}.
\end{align}
Since $S_c\in[0,|\mathcal{E}_c|]$, we obtain:
\begin{align}
    \mathrm{Var}(S_c) \le \frac{|\mathcal{E}_c|^2}{4}.
\end{align}
Let
\begin{align}
    T = K - |\mathcal{E}_p|\mu_p - |\mathcal{E}_n|\mu_n .
\end{align}
The budget is satisfied when:
\begin{align}
    |\mathcal{E}_p|\mu_p + |\mathcal{E}_n|\mu_n + S_c \le K,
\end{align}
which is equivalent to $S_c\le T$.

We next lower-bound $\Pr(S_c\le T)$. 
Cantelli's inequality states that, for any random variable $X$ with finite variance and any $a>0$:
\begin{align}
    \Pr(X-\mathbb{E}[X] \le a)
    \ge
    \frac{a^2}{\mathrm{Var}(X)+a^2}.
\end{align}
When $T>\mathbb{E}[S_c]$, we set:
\begin{align}
    a = T-\mathbb{E}[S_c].
\end{align}
Then Cantelli's inequality gives:
\begin{align}
    \Pr(S_c\le T)
    &=
    \Pr(S_c-\mathbb{E}[S_c]\le T-\mathbb{E}[S_c]) \\
    &\ge
    \frac{(T-\mathbb{E}[S_c])^2}
    {\mathrm{Var}(S_c)+(T-\mathbb{E}[S_c])^2}.
\end{align}
Using $\mathbb{E}[S_c]=|\mathcal{E}_c|\mu_c$ and 
$\mathrm{Var}(S_c)\le |\mathcal{E}_c|^2/4$, we further have:
\begin{align}
    \Pr(S_c\le T)
    \ge
    \frac{(T-|\mathcal{E}_c|\mu_c)^2}
    {\frac{|\mathcal{E}_c|^2}{4}+(T-|\mathcal{E}_c|\mu_c)^2}.
\end{align}

To ensure $\Pr(S_c\le T)\ge q$, it suffices that:
\begin{align}
    \frac{(T-|\mathcal{E}_c|\mu_c)^2}
    {\frac{|\mathcal{E}_c|^2}{4}+(T-|\mathcal{E}_c|\mu_c)^2}
    \ge q.
\end{align}
Since $T>|\mathcal{E}_c|\mu_c$, this condition is equivalent to:
\begin{align}
    T-|\mathcal{E}_c|\mu_c
    \ge
    \frac{|\mathcal{E}_c|}{2}
    \sqrt{\frac{q}{1-q}}.
\end{align}
Substituting the definition of $T$ yields:
\begin{align}
K
\ge
|\mathcal{E}_p|\mu_p
+
|\mathcal{E}_n|\mu_n
+
|\mathcal{E}_c|
\left(
\mu_c+\frac{1}{2}\sqrt{\frac{q}{1-q}}
\right),
\end{align}
which gives Eq.~\eqref{eq:budget_basic} by the definition of $c_q$.

Finally, since $\mathcal{E}_p$, $\mathcal{E}_n$, and $\mathcal{E}_c$ form a partition of $\mathcal{E}$:
\begin{align}
    |\mathcal{E}_c|
    =
    |\mathcal{E}|-|\mathcal{E}_p|-|\mathcal{E}_n|.
\end{align}
Substituting this into Eq.~\eqref{eq:budget_basic} gives:
\begin{align}
K
&\ge
|\mathcal{E}_p|\mu_p
+
|\mathcal{E}_n|\mu_n
+
\big(|\mathcal{E}|-|\mathcal{E}_p|-|\mathcal{E}_n|\big)c_q \\
&=
|\mathcal{E}|c_q
+
|\mathcal{E}_p|(\mu_p-c_q)
-
|\mathcal{E}_n|(c_q-\mu_n),
\end{align}
which gives Eq.~\eqref{eq:budget_tradeoff}.
\end{proof}

\textbf{Proposition 2} (Ranking correction by SD)\textbf{.}
\textit{Consider an important edge $e^+$ and an unimportant edge $e^-$. 
Let $m^+$ and $m^-$ denote their first-pass importance scores, and let $\Delta s^+$ and $\Delta s^-$ denote their instability scores under re-explanation. 
Suppose the original explanation ranks them incorrectly, i.e., $m^+ < m^-$, and the unimportant edge is more unstable, i.e.,}
\begin{align}
    \Delta s^- > \Delta s^+ .
\end{align}
\textit{Ignoring the clipping operation in Eq.~\eqref{eq:sd}, SD corrects their relative order, i.e., $\tilde{m}^+ > \tilde{m}^-$ whenever}
\begin{align}
    \eta >
    \frac{m^- - m^+}
    {m^- \Delta s^- - m^+ \Delta s^+}.
\end{align}

\begin{proof}
Ignoring the clipping operation in Eq.~\eqref{eq:sd}, SD updates the two edge scores as:
\begin{align}
    \tilde m^+
    =
    (1-\eta\Delta s^+)m^+,
    \qquad
    \tilde m^-
    =
    (1-\eta\Delta s^-)m^- .
\end{align}
To correct the ranking, we need $\tilde m^+>\tilde m^-$. This is equivalent to:
\begin{align}
    (1-\eta\Delta s^+)m^+
    >
    (1-\eta\Delta s^-)m^- .
\end{align}
Expanding both sides gives:
\begin{align}
    m^+ - \eta m^+\Delta s^+
    >
    m^- - \eta m^-\Delta s^- .
\end{align}
Rearranging terms, we obtain:
\begin{align}
    \eta
    \left(
        m^- \Delta s^- - m^+ \Delta s^+
    \right)
    >
    m^- - m^+ .
\end{align}
Since $m^- > m^+$ and $\Delta s^- > \Delta s^+$, we have:
\begin{align}
    m^- \Delta s^- - m^+ \Delta s^+ > 0 .
\end{align}
Therefore, dividing both sides by this positive quantity yields:
\begin{align}
    \eta >
    \frac{m^- - m^+}
    {m^- \Delta s^- - m^+ \Delta s^+}.
\end{align}
Thus, under this condition, SD reverses the incorrect ordering and gives $\tilde m^+>\tilde m^-$.
\end{proof}

\textbf{Proposition 3} (Prediction stability under SD)\textbf{.}
\textit{Let $G_s^{(1)} = G \odot M^{(1)}$ and $\tilde{G}_s = G \odot \tilde{M}$. 
Assume the prediction function $f$ is locally differentiable with respect to the edge-weight vector. 
If we require the prediction shift induced by SD to be bounded by $\varepsilon>0$, i.e.,}
\begin{align}
    |f(\tilde{G}_s)-f(G_s^{(1)})| \le \varepsilon,
\end{align}
\textit{then a sufficient condition is:}
\begin{align}
    \eta
    \le
    \frac{\varepsilon}
    {
    \|\nabla f(M^{(1)})\|_\infty
    \sum_{e_{ij}\in\mathcal{E}} m_{ij}^{(1)}\Delta s_{ij}
    }.
\end{align}

\begin{proof}
For clarity, we view the prediction function as a function of the edge-weight mask while keeping the input graph $G$ fixed. 
That is, we write:
\begin{align}
    F(M) := f(G \odot M).
\end{align}
Then
\begin{align}
    f(G_s^{(1)}) = F(M^{(1)}),
    \qquad
    f(\tilde{G}_s) = F(\tilde{M}).
\end{align}

Since $F$ is locally differentiable around $M^{(1)}$, a first-order approximation gives:
\begin{align}
    F(\tilde{M}) - F(M^{(1)})
    \approx
    \nabla F(M^{(1)})^\top(\tilde{M}-M^{(1)}).
\end{align}
Taking absolute values and applying H\"older's inequality yields:
\begin{align}
    |F(\tilde{M}) - F(M^{(1)})|
    &\le
    \left|
    \nabla F(M^{(1)})^\top(\tilde{M}-M^{(1)})
    \right| \\
    &\le
    \|\nabla F(M^{(1)})\|_\infty
    \cdot
    \|\tilde{M}-M^{(1)}\|_1 .
\end{align}
Here, $\|\nabla F(M^{(1)})\|_\infty$ denotes the maximum absolute partial derivative w.r.t. any edge weight:
\begin{align}
    \|\nabla F(M^{(1)})\|_\infty
    =
    \max_{e_{ij}\in\mathcal{E}}
    \left|
    \frac{\partial F}{\partial m_{ij}^{(1)}}(M^{(1)})
    \right|.
\end{align}
For notational simplicity, we write $\nabla f(M^{(1)})$ for $\nabla F(M^{(1)})$, which gives:
\begin{align}
    |f(\tilde{G}_s)-f(G_s^{(1)})|
    \le
    \|\nabla f(M^{(1)})\|_\infty
    \cdot
    \|M^{(1)}-\tilde{M}\|_1 .
\end{align}

Next, ignoring the clipping operation in Eq.~\eqref{eq:sd}, SD updates each edge score as:
\begin{align}
    \tilde{m}_{ij}
    =
    (1-\eta \Delta s_{ij})m_{ij}^{(1)} .
\end{align}
Therefore,
\begin{align}
    |m_{ij}^{(1)}-\tilde{m}_{ij}|
    &=
    \left|
    m_{ij}^{(1)}
    -
    (1-\eta \Delta s_{ij})m_{ij}^{(1)}
    \right| \\
    &=
    \eta m_{ij}^{(1)}\Delta s_{ij} .
\end{align}
Summing over all edges gives:
\begin{align}
    \|M^{(1)}-\tilde{M}\|_1
    =
    \sum_{e_{ij}\in\mathcal{E}}
    |m_{ij}^{(1)}-\tilde{m}_{ij}|
    =
    \eta
    \sum_{e_{ij}\in\mathcal{E}}
    m_{ij}^{(1)}\Delta s_{ij} .
\end{align}
Substituting this into the previous bound yields:
\begin{align}
    |f(\tilde{G}_s)-f(G_s^{(1)})|
    \le
    \eta
    \|\nabla f(M^{(1)})\|_\infty
    \sum_{e\in\mathcal{E}}
    m_{ij}^{(1)}\Delta s_{ij} .
\end{align}
This completes the proof.
\end{proof}

\textbf{Remark.} The local differentiability assumption is mild in the context of SI-GNNs. In practice, the prediction function $f(G \odot M)$ is implemented by neural networks composed of affine transformations and standard activations (e.g., ReLU), which are piecewise linear and thus differentiable almost everywhere with respect to the edge-weight mask $M$.  Moreover, the masking operation $G \odot M$ is linear in $M$, and therefore does not introduce additional non-smoothness. Although non-differentiabilities may arise at a measure-zero set (e.g., activation boundaries or clipping points), the result can be interpreted in a local or subgradient sense.

\begin{proposition}[Prediction stability under stochastic masking]\label{prop:sd_stochastic_stability}
Consider stochastic binary masks sampled from the corresponding soft masks:
\begin{align}
    X^{(1)} \sim \prod_{e_{ij}\in\mathcal{E}}\mathrm{Bern}(m_{ij}^{(1)}),
    \qquad
    \tilde X \sim \prod_{e_{ij}\in\mathcal{E}}\mathrm{Bern}(\tilde m_{ij}).
\end{align}
For binary classification, we have $0\le f(G)\le 1$ for any graph $G$. Then,
\begin{align}
\left|
\mathbb{E}_{\tilde X}\!\left[f(G\odot \tilde X)\right]
-
\mathbb{E}_{X^{(1)}}\!\left[f(G\odot X^{(1)})\right]
\right|
\le
\eta
\sum_{e_{ij}\in\mathcal{E}}
m_{ij}^{(1)}\Delta s_{ij} .
\end{align}
Consequently, to make the expected prediction shift no larger than $\varepsilon$, it suffices to choose:
\begin{align}
    \eta
    \le
    \frac{\varepsilon}
    {\sum_{e_{ij}\in\mathcal{E}}m_{ij}^{(1)}\Delta s_{ij}}.
\end{align}
\end{proposition}

\begin{proof}
Let $\mathbb{P}_{X^{(1)}}$ and $\mathbb{P}_{\tilde X}$ denote the probability distributions of the stochastic masks $X^{(1)}$ and $\tilde X$, respectively. 
Recall that the total variation distance between two distributions $\mathbb{P}$ and $\mathbb{Q}$ is defined as:
\begin{align}
    \mathrm{TV}(\mathbb{P},\mathbb{Q})
    :=
    \sup_{\mathcal{A}}
    |\mathbb{P}(\mathcal{A})-\mathbb{Q}(\mathcal{A})|,
\end{align}
where the supremum is taken over all measurable events $\mathcal{A}$. 

Intuitively, total variation measures the largest possible difference between the probabilities assigned by two distributions to the same event. A standard property of total variation distance (see, e.g.,~\cite{billingsley2013convergence}) states that, for any measurable function $g$ satisfying $0\le g\le 1$,
\begin{align}
    \left|
    \mathbb{E}_{X\sim\mathbb{P}}[g(X)]
    -
    \mathbb{E}_{Y\sim\mathbb{Q}}[g(Y)]
    \right|
    \le
    \mathrm{TV}(\mathbb{P},\mathbb{Q}).
    \label{eq:tv_function_bound}
\end{align}
In our setting, define:
\begin{align}
    g(X)=f(G\odot X).
\end{align}
Since we consider binary classification, the model output is a probability, and thus $0\le f(G)\le 1$ for any graph $G$. 
Therefore, $0\le g(X)\le 1$. 
Applying Eq.~\eqref{eq:tv_function_bound} with 
$\mathbb{P}=\mathbb{P}_{X^{(1)}}$ and 
$\mathbb{Q}=\mathbb{P}_{\tilde X}$ gives:
\begin{align}
\left|
\mathbb{E}_{\tilde X}\!\left[f(G\odot \tilde X)\right]
-
\mathbb{E}_{X^{(1)}}\!\left[f(G\odot X^{(1)})\right]
\right|
\le
\mathrm{TV}(\mathbb{P}_{X^{(1)}},\mathbb{P}_{\tilde X}).
\label{eq:tv_prediction_bound}
\end{align}

Next, we bound the total variation distance between the two mask distributions. 
Both $X^{(1)}$ and $\tilde X$ are binary masks sampled independently over edges. 
For each edge $e_{ij}$, the corresponding binary random variable follows a Bernoulli distribution:
\begin{align}
    X_{ij}^{(1)} \sim \mathrm{Bern}(m_{ij}^{(1)}),
    \qquad
    \tilde X_{ij} \sim \mathrm{Bern}(\tilde m_{ij}).
\end{align}
Thus, the full mask distributions are product Bernoulli distributions:
\begin{align}
    \mathbb{P}_{X^{(1)}}
    =
    \bigotimes_{e_{ij}\in\mathcal{E}}
    \mathrm{Bern}(m_{ij}^{(1)}),
    \qquad
    \mathbb{P}_{\tilde X}
    =
    \bigotimes_{e_{ij}\in\mathcal{E}}
    \mathrm{Bern}(\tilde m_{ij}).
\end{align}

We first consider a single edge $e_{ij}$. 
Let
\begin{align}
    \mathbb{P}_{ij}=\mathrm{Bern}(m_{ij}^{(1)}),
    \qquad
    \mathbb{Q}_{ij}=\mathrm{Bern}(\tilde m_{ij}).
\end{align}
A Bernoulli random variable only takes values in $\{0,1\}$. 
Therefore, the only nontrivial events are $\{1\}$ and $\{0\}$. 
For the event $\{1\}$, we have:
\begin{align}
    |\mathbb{P}_{ij}(\{1\})-\mathbb{Q}_{ij}(\{1\})|
    =
    |m_{ij}^{(1)}-\tilde m_{ij}|.
\end{align}
For the event $\{0\}$, we similarly have:
\begin{align}
    |\mathbb{P}_{ij}(\{0\})-\mathbb{Q}_{ij}(\{0\})|
    &=
    |(1-m_{ij}^{(1)})-(1-\tilde m_{ij})| \\
    &=
    |m_{ij}^{(1)}-\tilde m_{ij}|.
\end{align}
Hence,
\begin{align}
    \mathrm{TV}(\mathbb{P}_{ij},\mathbb{Q}_{ij})
    =
    |m_{ij}^{(1)}-\tilde m_{ij}|.
    \label{eq:tv_single_bernoulli}
\end{align}

Now we move from a single edge to the full mask. 
A standard tensorization bound for total variation distance (e.g., see \cite{lindvall2002lectures}) states that, for product distributions,
\begin{align}
    \mathrm{TV}
    \left(
    \bigotimes_{e_{ij}\in\mathcal{E}}\mathbb{P}_{ij},
    \bigotimes_{e_{ij}\in\mathcal{E}}\mathbb{Q}_{ij}
    \right)
    \le
    \sum_{e_{ij}\in\mathcal{E}}
    \mathrm{TV}(\mathbb{P}_{ij},\mathbb{Q}_{ij}).
    \label{eq:tv_tensorization}
\end{align}
Intuitively, the difference between two independently sampled masks is no larger than the sum of their edge-wise distribution differences. 
Combining Eq.~\eqref{eq:tv_single_bernoulli} and Eq.~\eqref{eq:tv_tensorization}, we obtain:
\begin{align}
    \mathrm{TV}(\mathbb{P}_{X^{(1)}},\mathbb{P}_{\tilde X})
    \le
    \sum_{e_{ij}\in\mathcal{E}}
    |m_{ij}^{(1)}-\tilde m_{ij}|.
    \label{eq:tv_sum_bound}
\end{align}

It remains to relate the edge-wise difference $|m_{ij}^{(1)}-\tilde m_{ij}|$ to the SD update. 
By definition, SD calibrates the edge score as:
\begin{align}
    \tilde m_{ij}
    =
    \max\left(0,(1-\eta\Delta s_{ij})m_{ij}^{(1)}\right).
\end{align}
Since the multiplicative factor is at most $1$, we have:
\begin{align}
    \tilde m_{ij} \le m_{ij}^{(1)}.
\end{align}
Therefore,
\begin{align}
    |m_{ij}^{(1)}-\tilde m_{ij}|
    =
    m_{ij}^{(1)}-\tilde m_{ij}.
\end{align}
We consider two cases.

If $(1-\eta\Delta s_e)m_e^{(1)}\ge 0$, then clipping is inactive and
\begin{align}
    \tilde m_{ij}
    =
    (1-\eta\Delta s_{ij})m_{ij}^{(1)}.
\end{align}
Thus,
\begin{align}
    m_{ij}^{(1)}-\tilde m_{ij}
    =
    \eta\;m_{ij}^{(1)}\Delta s_{ij}.
\end{align}

If $(1-\eta\Delta s_{ij})m_{ij}^{(1)}<0$, then clipping is active and $\tilde m_{ij}=0$. 
In this case,
\begin{align}
    m_{ij}^{(1)}-\tilde m_{ij}
    =
    m_{ij}^{(1)}.
\end{align}
The condition $(1-\eta\Delta s_{ij})m_{ij}^{(1)}<0$ implies $\eta\Delta s_{ij}>1$ whenever $m_{ij}^{(1)}>0$, and therefore
\begin{align}
    m_{ij}^{(1)}
    \le
    \eta m_{ij}^{(1)}\Delta s_{ij}.
\end{align}
If $m_{ij}^{(1)}=0$, the inequality holds trivially. 
Hence, in all cases,
\begin{align}
    |m_{ij}^{(1)}-\tilde m_{ij}|
    \le
    \eta m_{ij}^{(1)}\Delta s_{ij}.
    \label{eq:sd_edge_bound}
\end{align}

Substituting Eq.~\eqref{eq:sd_edge_bound} into Eq.~\eqref{eq:tv_sum_bound}, we obtain:
\begin{align}
    \mathrm{TV}(\mathbb{P}_{X^{(1)}},\mathbb{P}_{\tilde X})
    \le
    \eta
    \sum_{e_{ij}\in\mathcal{E}}
    m_{ij}^{(1)}\Delta s_{ij}.
    \label{eq:tv_sd_bound}
\end{align}
Finally, combining Eq.~\eqref{eq:tv_prediction_bound} and Eq.~\eqref{eq:tv_sd_bound} gives:
\begin{align}
\left|
\mathbb{E}_{\tilde X}\!\left[f(G\odot \tilde X)\right]
-
\mathbb{E}_{X^{(1)}}\!\left[f(G\odot X^{(1)})\right]
\right|
\le
\eta
\sum_{e_{ij}\in\mathcal{E}}
m_{ij}^{(1)}\Delta s_{ij}.
\end{align}

Therefore, to ensure that the expected prediction shift is at most $\varepsilon$, it suffices to require:
\begin{align}
    \eta
    \sum_{e_{ij}\in\mathcal{E}}
    m_{ij}^{(1)}\Delta s_{ij}
    \le
    \varepsilon,
\end{align}
which is equivalent to
\begin{align}
    \eta
    \le
    \frac{\varepsilon}
    {\sum_{e_{ij}\in\mathcal{E}}m_{ij}^{(1)}\Delta s_{ij}}.
\end{align}
This completes the proof.
\end{proof}

\textbf{Remark.} In most SI-GNNs, discrete explanatory subgraphs are obtained via stochastic sampling (e.g., Gumbel-Sigmoid) during training, while deterministic soft masks are used at inference time for stability. This discrepancy leads to a mismatch between the training and inference distributions of explanations. To account for this, we additionally analyze prediction stability under stochastic masking by modeling explanations as independent Bernoulli samples and deriving an expectation-based bound. \cref{prop:sd_stochastic_stability} complements the deterministic analysis in Proposition~\ref{prop:sd_stability}.

\section{Implementation Details}\label{app:setting}

\subsection{SI-GNN Pipelines}

A typical SI-GNN consists of three components: an explainer, a GNN encoder, and a classifier. Given an input graph $G = (\mathcal{V}, \mathcal{E})$, the encoder first computes node representations:
\begin{align}
    \mathbf{H} = h_Z(G),
\end{align}
where $\mathbf{H} = \{\mathbf{h}_i\}_{i\in\mathcal{V}}$ denotes node embeddings.

Based on these embeddings, the explainer predicts a soft edge mask $M \in [0,1]^{|\mathcal{E}|}$, where each entry $m_{ij}$ reflects the importance of edge $e_{ij}$. A common parameterization is:
\begin{align}
    m_{ij} = \sigma\!\left( \mathrm{MLP}([\mathbf{h}_i; \mathbf{h}_j]) \right),
\end{align}
where $[\cdot;\cdot]$ denotes concatenation and $\sigma(\cdot)$ is the sigmoid function.
During training, discrete explanatory subgraphs are typically obtained via stochastic sampling. Following prior work~\cite{miao2022interpretable,luo2024towards}, this is implemented using the Gumbel--Sigmoid trick~\cite{jang2017categorical} to enable differentiable sampling:
\begin{align}
    {m}_{ij} := \sigma\!\left( \frac{\log \epsilon - \log(1-\epsilon) + m_{ij}}{\tau} \right),
\end{align}
where $\epsilon \sim \mathrm{Uniform}(0,1)$ and $\tau$ is a temperature parameter. At inference time, the deterministic mask $M$ is used directly.
The explanatory graph is then constructed as $G_s = G \odot M$, and fed back into the encoder to obtain a graph-level representation:
\begin{align}
    \mathbf{z}_s = \mathrm{POOL}\!\left( h_Z(G_s) \right),
\end{align}
where $\mathrm{POOL}$ is a permutation-invariant readout function. Finally, a classifier produces the prediction:
\begin{align}
    \hat{Y} = h_{\hat{Y}}(\mathbf{z}_s).
\end{align}
\textbf{Remark on CAL.} CAL~\cite{sui2022causal} differs from the standard SI-GNN pipeline in how it uses the explanatory subset for prediction. CAL employs three classifiers corresponding to the three terms in \cref{eq:type_2}: the first takes $G_s$ as input, the second takes $\bar{G}_s$, and the third takes the intervened graph $G_s \cup \bar{G}_s'$ as input. During training, CAL applies intervention to $\bar{G}_s$, whereas no intervention is applied at inference time. Therefore, the third classifier effectively receives the full graph $G$ at inference, and its output is used as CAL's final prediction. Despite this difference, CAL still computes edge importance scores from node representations produced by message passing. Therefore, re-explanation can perturb these representations and change the resulting edge scores, so the context-perturbation analysis in \cref{subsec:local-context} remains applicable. However, this design choice affects the reliability of fidelity evaluation, which we will discuss below.

\subsection{Experimental Settings}

\begin{table}[t]
    \centering
    \caption{Datasets and their URLs.}
    \label{tab:data_url}
    \resizebox{\textwidth}{!}{
    \begin{tabular}{|c|c|}
    \hline
    DATASET & URL \\
    \hline
    BA-2MOTIFS & \url{https://github.com/Graph-COM/GSAT?tab=readme-ov-file}\\
    \hline
    3MR & \url{https://drive.google.com/drive/folders/1b0MowzK4LSlkih3ie1bnj6IkjRqPAKH5}\\
    \hline
    BENZENE & \url{https://chrsmrrs.github.io/datasets/docs/datasets}\\
    \hline
    MUTAGENICITY & \url{https://github.com/flyingdoog/PGExplainer/tree/master/dataset}\\
    \hline
    \end{tabular}
    }
\end{table}

Our proposed SD is a training-free post-processing strategy. Throughout all experiments, we directly use the pretrained checkpoints released by Tai et al.~\cite{tai2025redundancy}, where 10 models are trained under different random seeds for each SI-GNN. The only stage involving additional optimization is the lightweight classifier adaptation used for selecting the denoising strength $\eta$ and mitigating the induced distribution shift (SD*). Following Tai et al.~\cite{tai2025redundancy}, after loading the pretrained checkpoints, we freeze the explainer and GNN encoder and fine-tune only the classifier for 10 epochs. 

In practice, we find that hyperparameter selection can also be simplified: moderate values of $\eta$ (e.g., $\eta=0.5$ or $\eta=1$) consistently yield good performance across different settings, providing reasonable default choices without extensive tuning.

For completeness, we summarize the SI-GNN training settings below, although these settings are inherited directly from prior work rather than introduced by SD itself. All SI-GNN baselines are implemented within a unified framework to ensure fair and reproducible comparison. Each model adopts a two-layer GIN encoder~\citep{xu2019powerful} (hidden dimensions: 64, 64; dropout rate: 0.3). The explainer is a three-layer MLP (hidden dimensions: 256, 64, 1; dropout rate: 0.5) that predicts edge importance scores, and the classifier is another three-layer MLP (hidden dimensions: 64, 64, 1) for final graph-level prediction. The only exception is GSAT on the BA-2MOTIFS dataset, for which we observe that setting $\beta = 0.05$ (instead of $\beta = 1$ as in Tai et al.~\cite{tai2025redundancy}) yields higher validation accuracy and consistently better results. Since EE can aggregate an arbitrary number of models (up to ten in our setup), we ensemble five models per SI-GNN type to balance accuracy and efficiency. During inference, soft edge scores are used as explanations.

\textbf{Computing resources.} All experiments were conducted using PyTorch and executed on one NVIDIA RTX 4090 GPU with Intel Core i7-13700KF CPU. The NVIDIA driver version is 545.23.06, and the CUDA version is 12.3. Compared to a standard SI-GNN, SD requires only 4--6\% more time. 

\textbf{Open Access to Data and Code.} All datasets are published and can be downloaded from the Internet (see \cref{tab:data_url}). 
Code will be publicly released in a future version.

\subsection{Metrics}\label{app:metrics}

We evaluate explanation quality from four complementary perspectives: accuracy, informativeness, conciseness, and faithfulness.

\textbf{Accuracy (AUC).} We adopt the ROC-AUC metric to measure how well the model distinguishes truly explanatory edges from irrelevant ones~\cite{wu2022discovering,miao2022interpretable,rao2022quantitative,luo2024towards}. AUC quantifies the agreement between the generated importance scores and ground-truth annotations. Following~\cite{faber2021comparing}, we note that the validity of AUC depends on the reliability of the provided ground-truth explanations; all datasets used in this work have been independently verified to align with the decision rationale of SI-GNNs~\cite{tai2025redundancy}.

\textbf{Informativeness (ACC).} We further report the classification accuracy when predictions are made solely on the explanatory subgraph $G_s$~\cite{wu2022discovering,amara2022graphframex,luo2024towards}. This metric reflects how informative the identified subset is—if the explanation retains task-relevant information, the downstream classification performance should be high.

\textbf{Conciseness (SPA).} To quantify how compact and human-readable an explanation is, we introduce Soft Sparsity (SPA), which measures the average activation level of edge importance scores:
\begin{align}
    \mathrm{SPA}(G, G_s) = \frac{1}{|\mathcal{E}|} \sum_{e_{ij} \in \mathcal{E}} m_{ij},
\end{align}
Due to the design principle of SD, our strategy inevitably leads to a decrease in the SPA metric. Despite that, SPA serves as a meaningful indicator of explanation quality—the sharper distinction between important and unimportant edges reflects improved human interpretability.

\textbf{Faithfulness (FID).} Faithfulness evaluates whether the explanatory subgraph $G_s$ truly supports the model’s original decision~\citep{yuan2022explainability}. We adopt Fidelity$^-$ (FID), defined as:
\begin{align}
\mathrm{FID}^{-}(G,G_s) = 1 - \mathds{1}\big(c(f(G_s)) = c(f(G))\big),
\end{align}
where $c(\cdot)$ denotes the predicted class. A lower FID implies that the model's prediction remains stable when restricted to the explanatory subset, indicating higher faithfulness. Another commonly used metric is Fidelity$^+$ (FID$^+$), which measures whether the prediction changes after removing the explanatory subgraph:
\begin{align}
\mathrm{FID}^{+}(G,G_s) = 1 - \mathds{1}\big(c(f(G \setminus G_s)) = c(f(G))\big).
\end{align}
Although fidelity metrics are widely used for evaluating explanation faithfulness, they are perturbation-based by nature and may introduce distribution shift during evaluation. As a result, a low fidelity score does not necessarily indicate poor explanation quality; it may also reflect limited model generalization under perturbed graph distributions~\cite{faber2021comparing}. This issue is particularly severe for FID$^+$, because it evaluates the model on the complement graph $G\setminus G_s$, which is never used to train the GNN encoder. In contrast, FID$^-$ is relatively less affected, since the encoder has at least been exposed to explanatory subgraphs $G_s$ during training. Nevertheless, both FID$^-$ and FID$^+$ can still be influenced by OOD effects~\citep{tai2025redundancy,azzolin2025reconsidering,tai2026self}. 

Besides, fidelity metrics can be biased toward redundant explanations. Since FID$^-$ evaluates whether the selected subgraph alone preserves the prediction, explanations containing many redundant edges can obtain low FID even if they are not concise or plausible. For example, GAT on MUTAGENICITY achieves a very low FID of $0.37\%$, but its sparsity is as high as $88.66\%$, indicating that the explanation still contains many edges. This is directly misaligned with the goal of SD, which aims to removes unimportant edges and improves explanation conciseness. 

Therefore, fidelity-based evaluation is inherently less suitable for assessing denoising-oriented calibration methods such as ours. Although we report and analyze fidelity results below for completeness, our main evaluation of explanation quality relies on AUC computed against reliable human-annotated explanations. This choice follows the recommendation of prior work that, when trustworthy ground-truth explanations are available, they provide a more direct basis for evaluating explanation plausibility~\citep{faber2021comparing}. The reliability of the ground-truth explanations used in our benchmark datasets has also been validated in prior SI-GNN studies~\citep{tai2025redundancy}.

\textbf{Remark.} We observe that FID can be particularly unstable for CAL across all GNN backbones when SD is applied without classifier adaptation (see \cref{tab:main_results,tab:main_results_graphsage,tab:with_EE_gatedgcn}). We attribute this to the special prediction mechanism of CAL. Unlike standard SI-GNNs that make predictions directly from the explanatory subgraph, CAL uses the full graph for final prediction at inference time. Therefore, reducing edge scores through SD does not necessarily remove decision-relevant evidence in the same way as in other SI-GNNs. Instead, it may merely introduce a stronger distribution shift, making the fidelity evaluation less stable and sometimes worse. After lightweight classifier adaptation, this issue is alleviated, which supports our analysis that the degradation is mainly caused by distribution shift rather than by a loss of explanation quality.


\section{Additional Results}\label{app:results}

\subsection{Selecting $\eta$ via adapted predictive performance}

\begin{figure}[t]    
    \centering    
    \includegraphics[width=\linewidth]{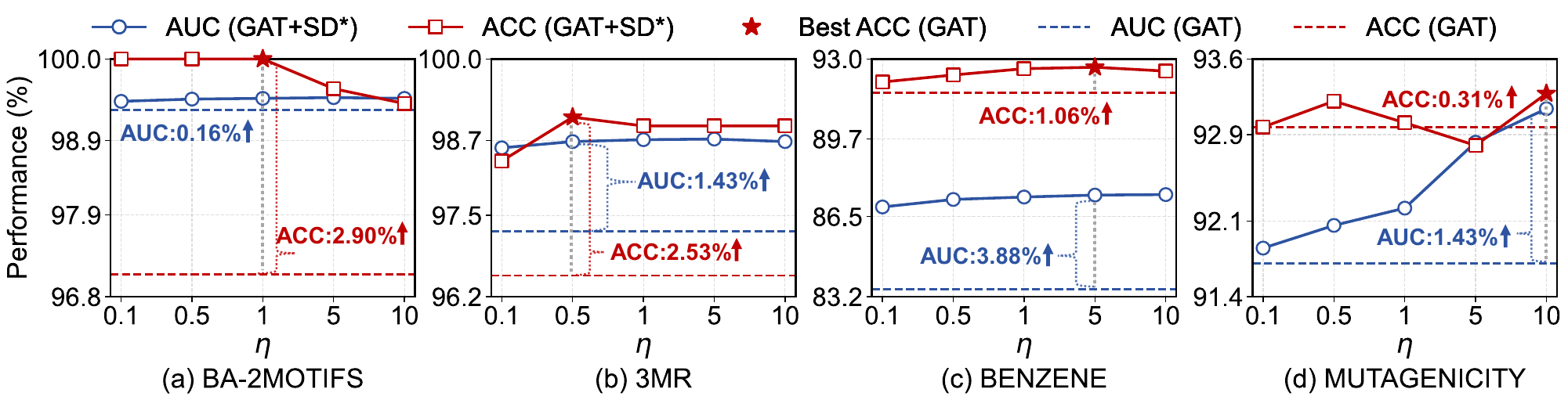}    
    \caption{Selecting $\eta$ via adapted predictive performance on GAT.}
    \label{fig:eta-gat}   
\end{figure}

\begin{figure}[t]    
    \centering    
    \includegraphics[width=\linewidth]{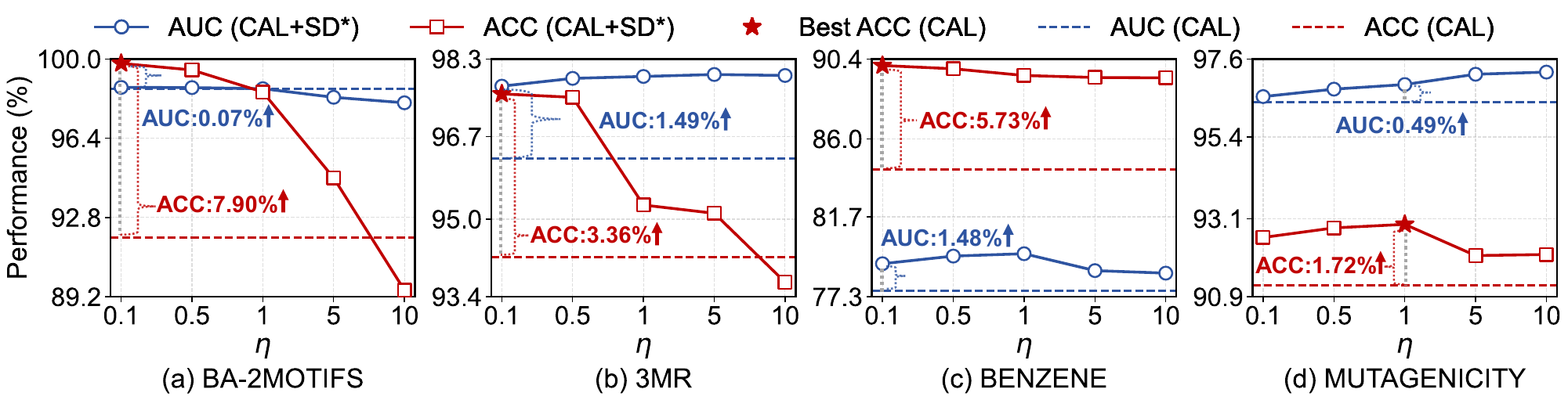}    
    \caption{Selecting $\eta$ via adapted predictive performance on CAL.}
    \label{fig:eta-cal}   
\end{figure}

\begin{figure}[t]    
    \centering    
    \includegraphics[width=\linewidth]{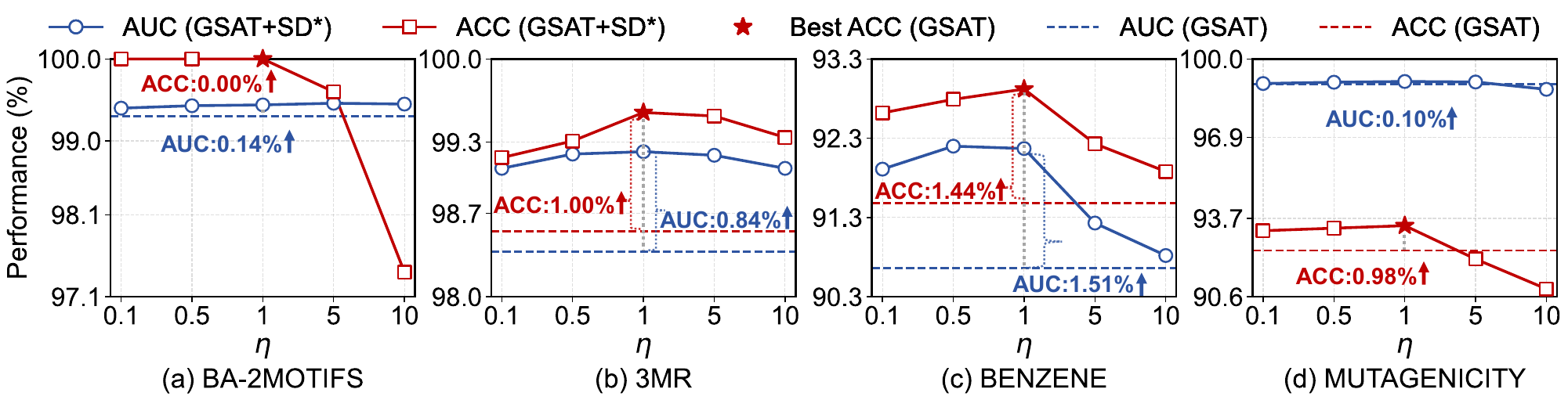}    
    \caption{Selecting $\eta$ via adapted predictive performance on GSAT.}
    \label{fig:eta-gsat}   
\end{figure}

In \cref{subsec:eta-select}, we proposed selecting the denoising strength $\eta$ based on validation accuracy after lightweight classifier adaptation. In the main text, we illustrated this strategy using SMGNN. Here, we provide the corresponding results for the other three SI-GNNs. As shown in \cref{fig:eta-gat}-\cref{fig:eta-gsat}, similar trends consistently appear across different architectures. Moderate values of $\eta$ typically improve predictive performance after adaptation, suggesting that SD removes context-driven noise while preserving task-relevant information. In contrast, excessively large $\eta$ leads to performance degradation, indicating over-denoising and information loss. These results further support the feasibility and generality of using adapted predictive performance to select $\eta$ in practice.

\subsection{Sparsity Analysis}

\begin{figure}[t]    
    \centering    
    \includegraphics[width=\linewidth]{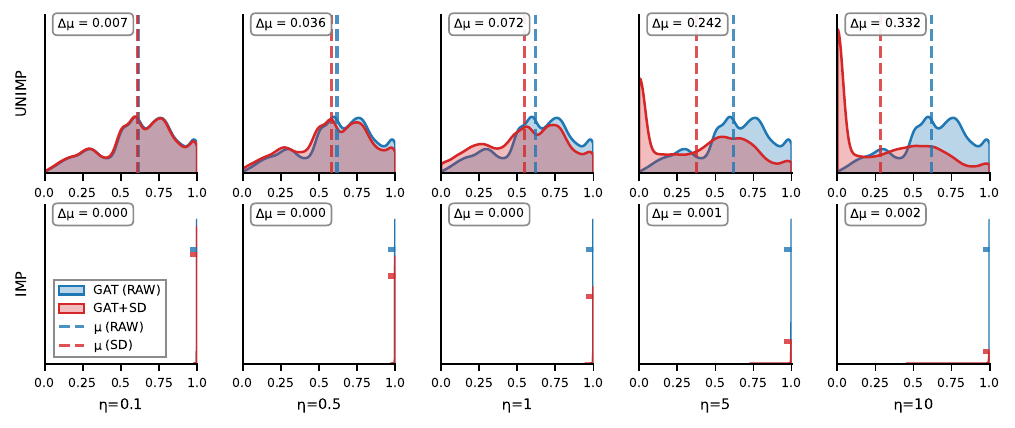}    
    \caption{Edge-score distributions of GAT under different $\eta$ on the BA-2MOTIFS dataset.} 
    \label{fig:sparsity-1}
\end{figure}

\begin{figure}[t]    
    \centering    
    \includegraphics[width=\linewidth]{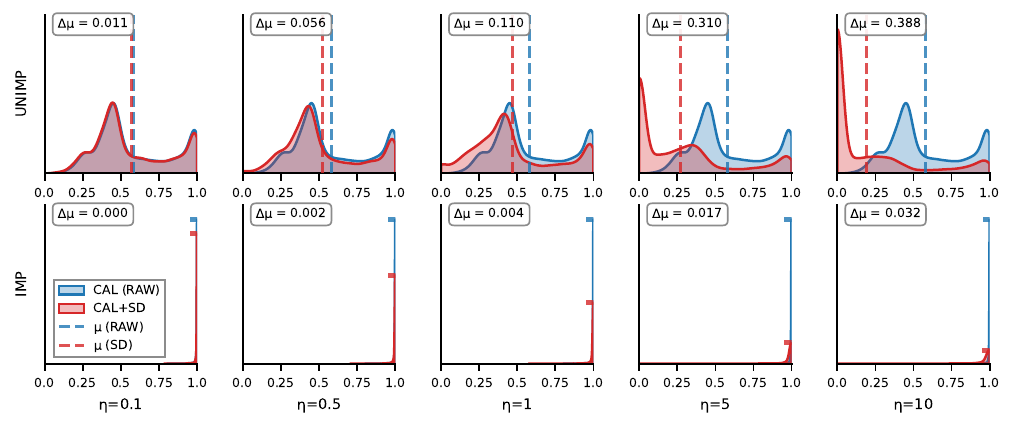}    
    \caption{Edge-score distributions of CAL under different $\eta$ on the BA-2MOTIFS dataset.} 
\end{figure}

\begin{figure}[t]    
    \centering    
    \includegraphics[width=\linewidth]{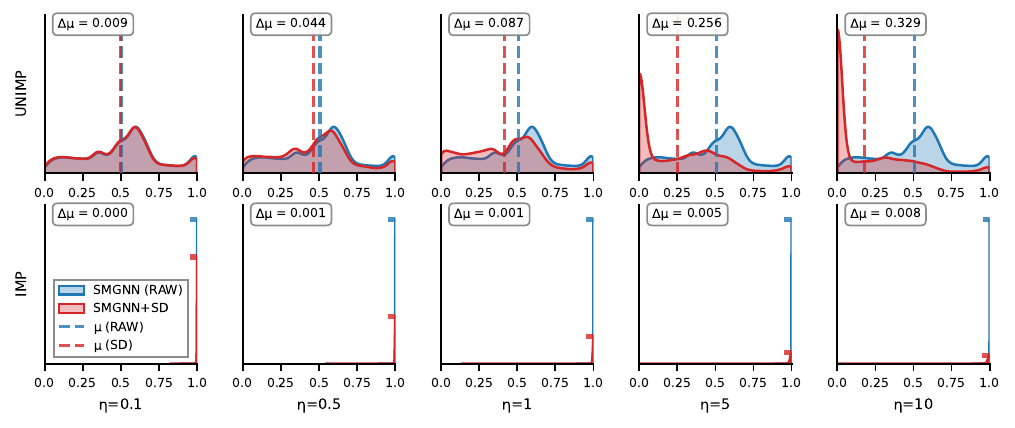}    
    \caption{Edge-score distributions of SMGNN under different $\eta$ on the BA-2MOTIFS dataset.} 
\end{figure}

\begin{figure}[t]    
    \centering    
    \includegraphics[width=\linewidth]{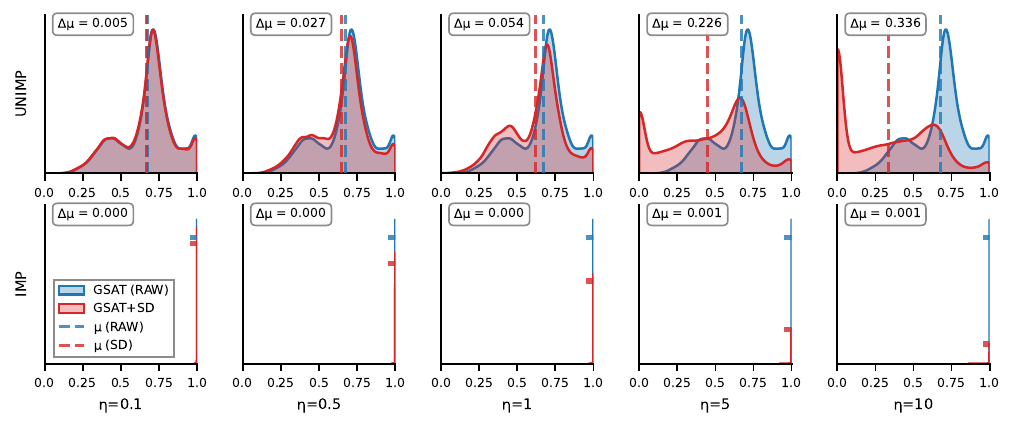}    
    \caption{Edge-score distributions of GSAT under different $\eta$ on the BA-2MOTIFS dataset.} 
\end{figure}

\begin{figure}[t]    
    \centering    
    \includegraphics[width=\linewidth]{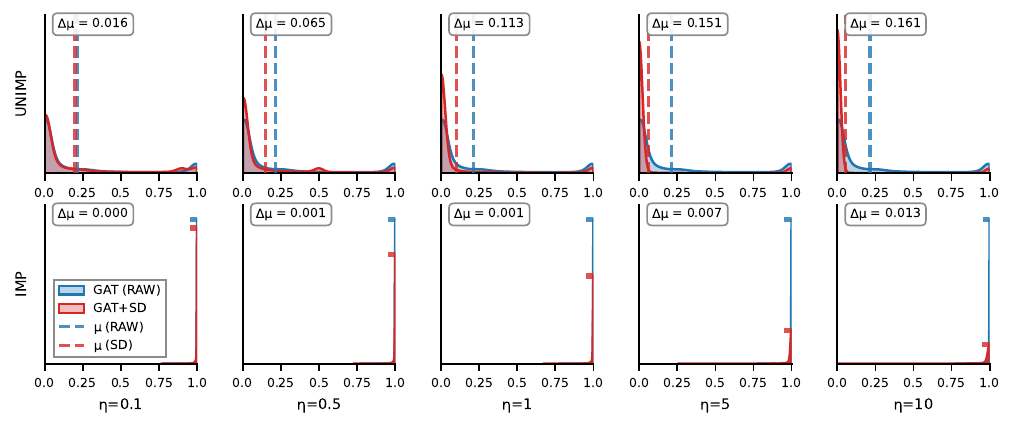}    
    \caption{Edge-score distributions of GAT under different $\eta$ on the 3MR dataset.} 
\end{figure}

\begin{figure}[t]    
    \centering    
    \includegraphics[width=\linewidth]{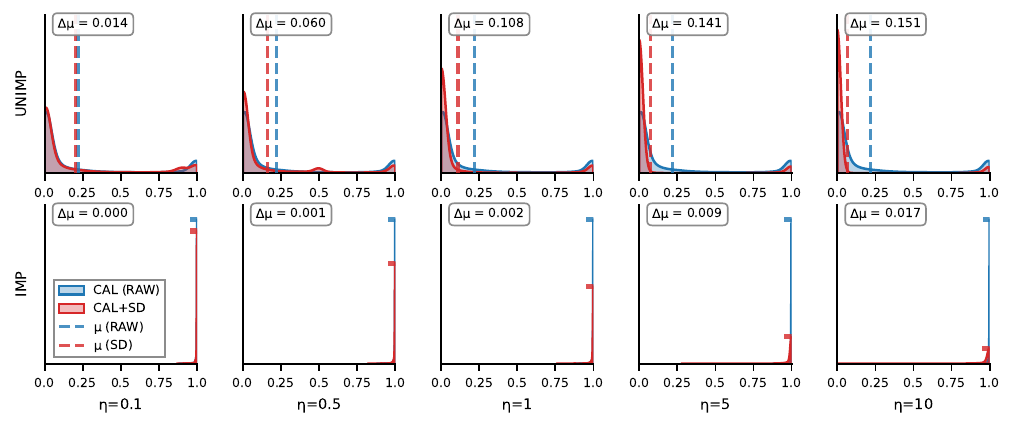}    
    \caption{Edge-score distributions of CAL under different $\eta$ on the 3MR dataset.} 
\end{figure}

\begin{figure}[t]    
    \centering    
    \includegraphics[width=\linewidth]{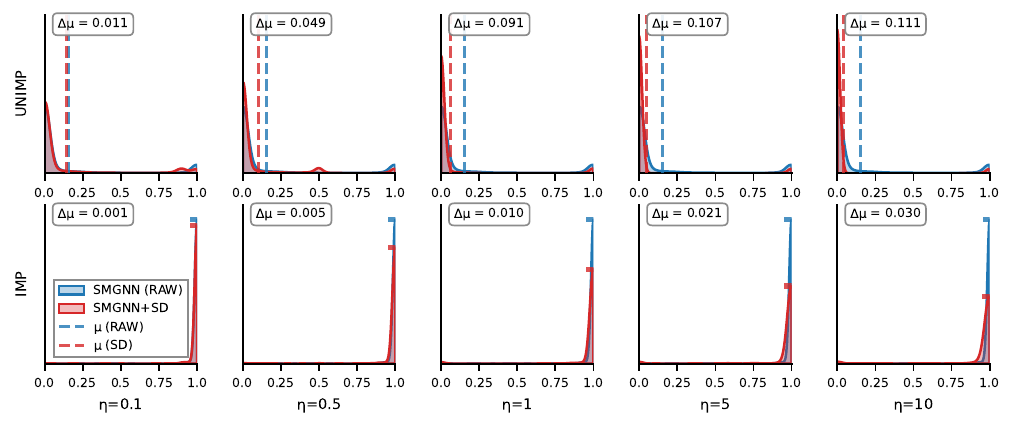}    
    \caption{Edge-score distributions of SMGNN under different $\eta$ on the 3MR dataset.} 
\end{figure}

\begin{figure}[t]    
    \centering    
    \includegraphics[width=\linewidth]{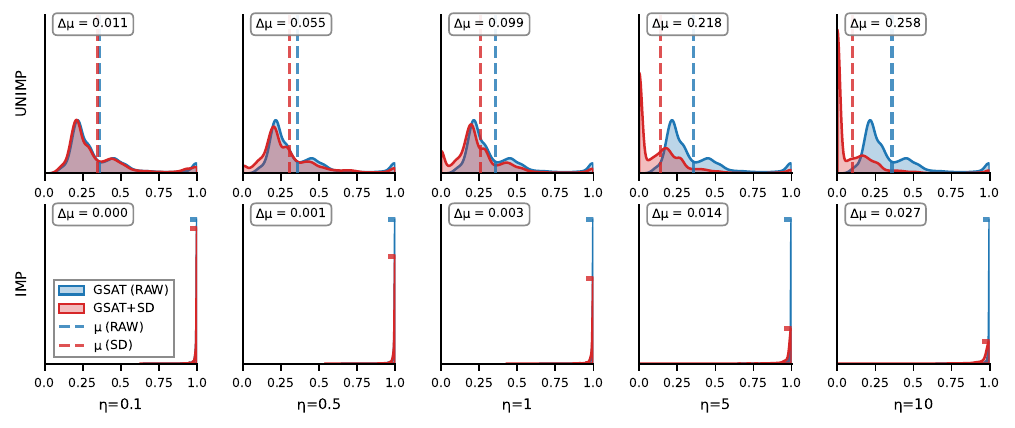}    
    \caption{Edge-score distributions of GSAT under different $\eta$ on the 3MR dataset.} 
\end{figure}

\begin{figure}[t]    
    \centering    
    \includegraphics[width=\linewidth]{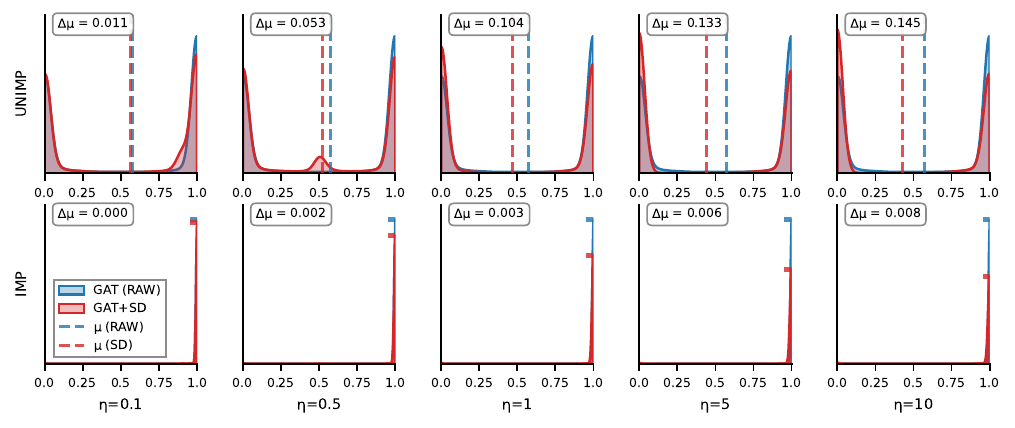}    
    \caption{Edge-score distributions of GAT under different $\eta$ on the BENZENE dataset.} 
\end{figure}

\begin{figure}[t]    
    \centering    
    \includegraphics[width=\linewidth]{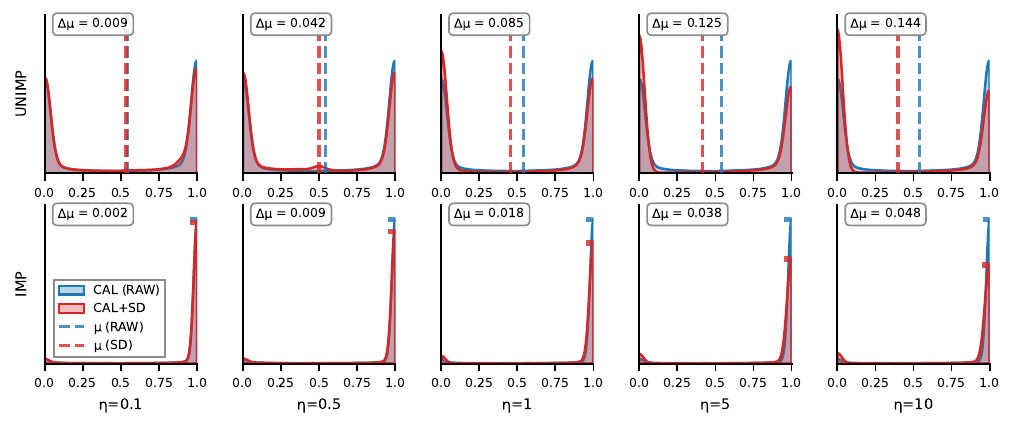}    
    \caption{Edge-score distributions of CAL under different $\eta$ on the BENZENE dataset.} 
\end{figure}

\begin{figure}[t]    
    \centering    
    \includegraphics[width=\linewidth]{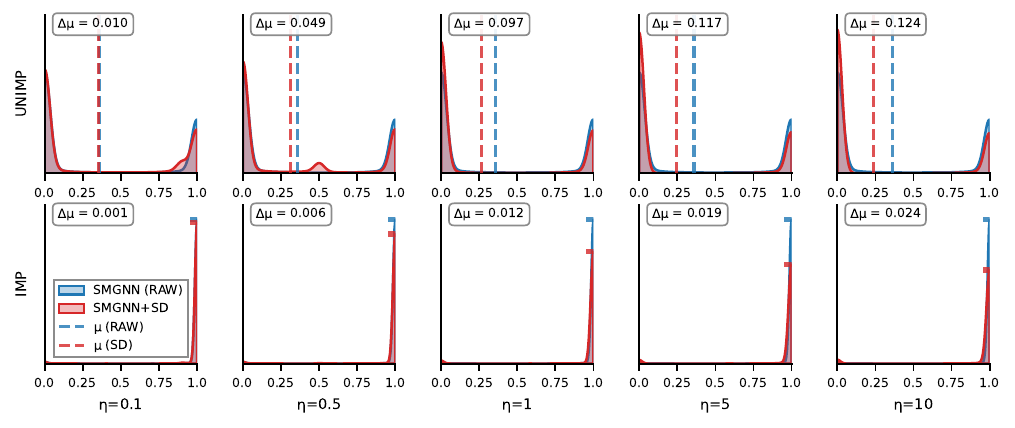}    
    \caption{Edge-score distributions of SMGNN under different $\eta$ on the BENZENE dataset.} 
\end{figure}

\begin{figure}[t]    
    \centering    
    \includegraphics[width=\linewidth]{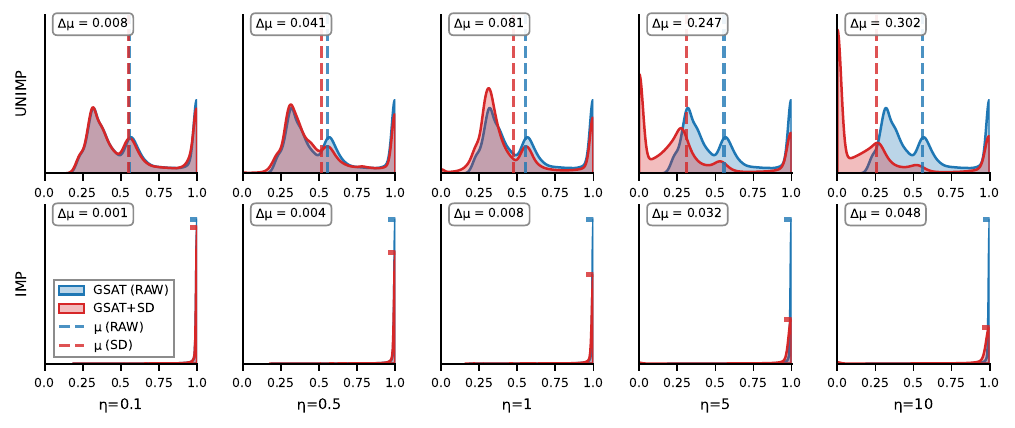}    
    \caption{Edge-score distributions of GSAT under different $\eta$ on the BENZENE dataset.} 
\end{figure}

\begin{figure}[t]    
    \centering    
    \includegraphics[width=\linewidth]{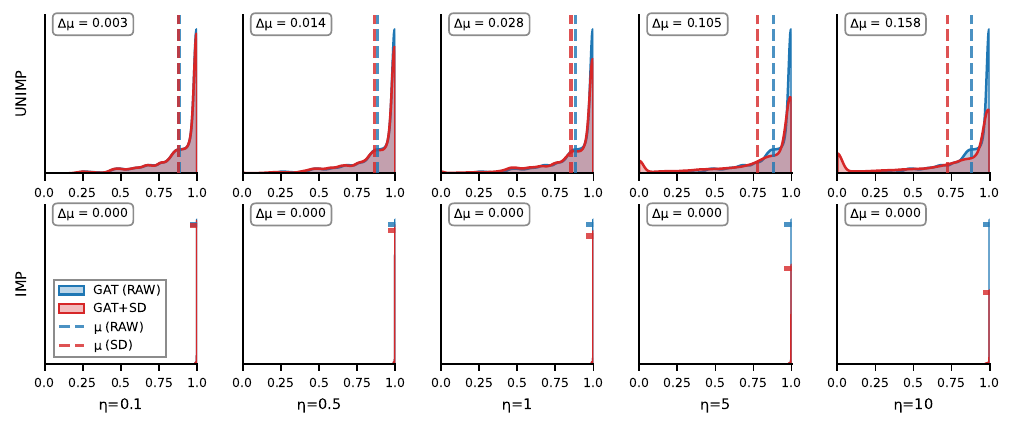}    
    \caption{Edge-score distributions of GAT under different $\eta$ on the MUTAGENICITY dataset.} 
\end{figure}

\begin{figure}[t]    
    \centering    
    \includegraphics[width=\linewidth]{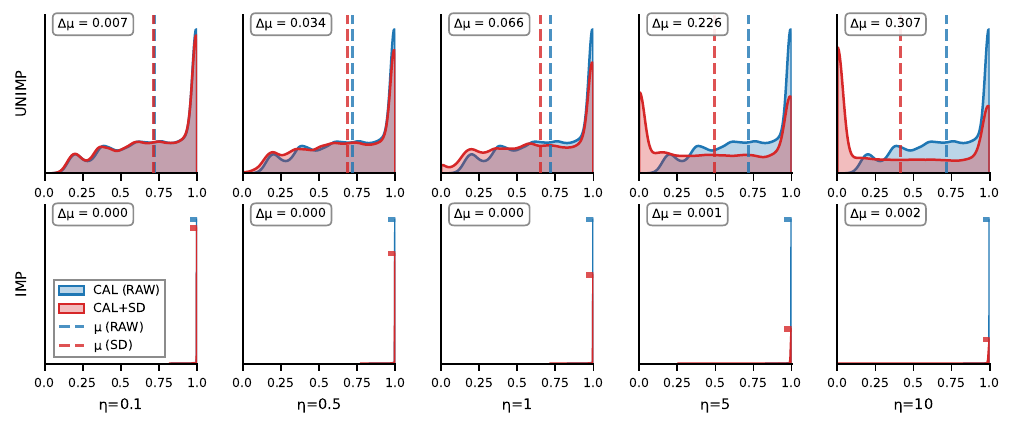}    
    \caption{Edge-score distributions of CAL under different $\eta$ on the MUTAGENICITY dataset.} 
\end{figure}

\begin{figure}[t]    
    \centering    
    \includegraphics[width=\linewidth]{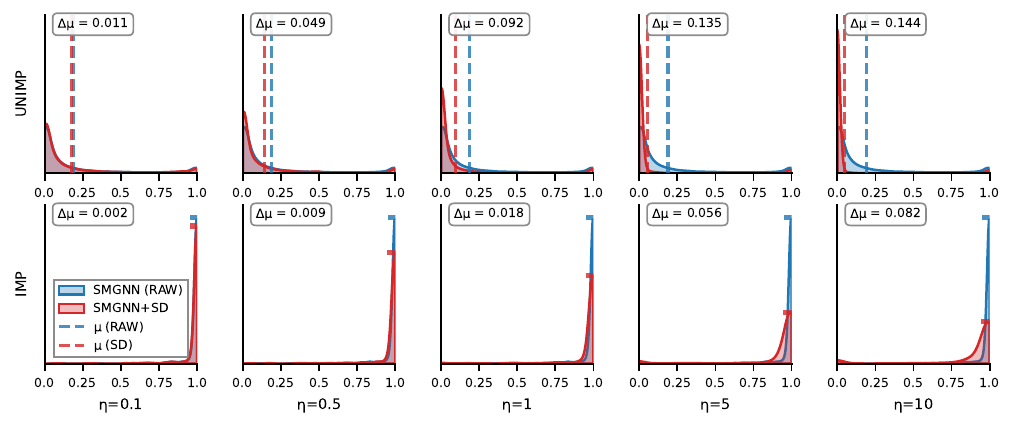}    
    \caption{Edge-score distributions of SMGNN under different $\eta$ on the MUTAGENICITY dataset.} 
\end{figure}

\begin{figure}[t]    
    \centering    
    \includegraphics[width=\linewidth]{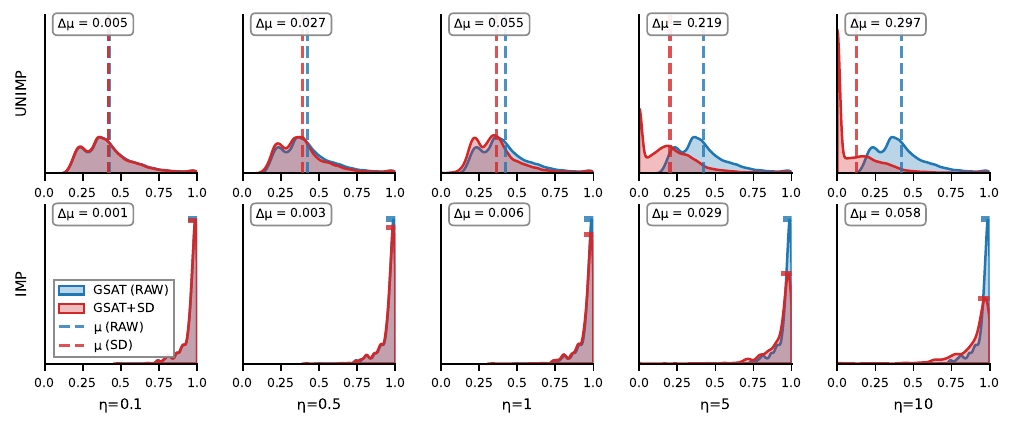}    
    \caption{Edge-score distributions of GSAT under different $\eta$ on the MUTAGENICITY dataset.}
    \label{fig:sparsity-16}
\end{figure}

Since SD always decreases edge scores, it naturally reduces SPA. The key question is whether this reduction mainly affects noisy edges or also suppresses truly important structures. As shown in \cref{fig:sparsity-1}-\cref{fig:sparsity-16}, SD primarily reduces the scores of unimportant edges, while the scores of important edges remain largely stable. This behavior makes the separation between important and unimportant edges clearer, leading to more concise explanations that are easier for humans to interpret and use. As $\eta$ increases, the denoising effect becomes stronger. AUC decreases only when $\eta$ becomes excessively large. This happens because not all unimportant edges can be identified through self-inconsistency, while some important edges may also exhibit mild instability due to factors such as budget tightening or imperfect signal allocation. Therefore, an overly large $\eta$ may over-suppress informative edges and rank them below self-consistent but unimportant ones. In practice, this issue can be mitigated by the hyperparameter selection strategy discussed in \cref{subsec:eta-select}, which controls the strength of explanation calibration based on validation performance.

\subsection{Comparison with Training-Time Self-Consistency Regularization}\label{app:with-sc-iclr}

\begin{table}[t]      
\caption{Comparison between the proposed SD and training-time self-consistency regularization (SC) on SI-GNNs with the GIN backbone across four datasets. SC denotes self-consistency training~\cite{tai2026self}, while SC* indicates variants with additional conciseness regularization for stability. \textbf{Bold} indicates better performance, while \underline{underlined} results are statistically significant ($p<0.05$).}
\label{tab:with-sc-iclr}    
\begin{center}    
\setlength{\tabcolsep}{2pt} 
\resizebox{\linewidth}{!}{    
\begin{tabular}{ccccccccc}
\toprule        
\multirow{2}{*}{\vspace{-2mm} Method}
& \multicolumn{2}{c}{BA-2MOTIFS}         
& \multicolumn{2}{c}{3MR}         
& \multicolumn{2}{c}{BENZENE}         
& \multicolumn{2}{c}{MUTAGENICITY} \\        
\cmidrule(lr){2-3}\cmidrule(lr){4-5}\cmidrule(lr){6-7}\cmidrule(lr){8-9}        
& $\uparrow$ AUC (\%) & $\downarrow$ FID (\%)
& $\uparrow$ AUC (\%) & $\downarrow$ FID (\%)         
& $\uparrow$ AUC (\%) & $\downarrow$ FID (\%)         
& $\uparrow$ AUC (\%) & $\downarrow$ FID (\%) \\        
\midrule        
GAT~\citep{velivckovic2018graph}        
& 99.31$\pm$0.35 & 2.50$\pm$7.50
& 97.25$\pm$0.67 & {2.63$\pm$1.93}
& 83.51$\pm$2.24 & {2.06$\pm$0.65}
& 91.71$\pm$6.00 & 0.37$\pm$0.53 \\                
GAT+SD
& \textbf{99.47$\pm$0.28} & \underline{\textbf{0.80$\pm$2.40}}
& \underline{\textbf{98.68$\pm$0.40}} & {\textbf{2.15$\pm$1.17}}
& \underline{\textbf{87.39$\pm$1.57}} & \underline{\textbf{1.24$\pm$0.73}}
& \underline{\textbf{93.14$\pm$5.29}} & 1.22$\pm$1.14 \\ 
GAT+SD*
& \textbf{99.47$\pm$0.28} & \underline{\textbf{0.00$\pm$0.00}}
& \underline{\textbf{98.68$\pm$0.40}} & 6.37$\pm$4.41
& \underline{\textbf{87.39$\pm$1.57}} & \underline{\textbf{1.56$\pm$0.93}}
& \underline{\textbf{93.14$\pm$5.29}} & 1.25$\pm$1.34 \\
GAT+SC~\citep{tai2026self}        
& 99.14$\pm$2.15 & \underline{\bf0.00$\pm$0.00}
& \bf97.66$\pm$1.00 & \underline{\bf0.10$\pm$0.00}
& \bf85.13$\pm$5.09 & \underline{0.82$\pm$0.61}
& 81.79$\pm$12.31 & \underline{\bf0.00$\pm$0.00} \\ 
GAT+SC*~\citep{tai2026self}        
& \underline{\bf99.87$\pm$0.05} & \underline{\bf0.00$\pm$0.00}
& \underline{\bf98.39$\pm$0.88} & \underline{\bf0.00$\pm$0.00}
& \underline{\bf90.07$\pm$1.56} & \underline{\bf0.97$\pm$0.59}
& \underline{\bf98.30$\pm$0.54} & 0.61$\pm$0.54 \\ 
\midrule        
CAL~\citep{sui2022causal}               
& 98.64$\pm$1.33 & 17.10$\pm$17.21
& 96.25$\pm$1.59 & 8.17$\pm$2.52
& 77.64$\pm$3.05 & 4.97$\pm$3.44
& 96.39$\pm$1.39 & 2.06$\pm$1.14 \\  
CAL+SD
& \textbf{98.71$\pm$1.25} & 17.60$\pm$17.04
& \underline{\textbf{97.74$\pm$1.13}} & 8.24$\pm$2.54
& \textbf{79.12$\pm$3.05} & {4.90$\pm$3.52}
& \textbf{96.88$\pm$1.21} & 2.47$\pm$1.35 \\
CAL+SD*
& \textbf{98.71$\pm$1.25} & \underline{\textbf{1.50$\pm$2.38}}
& \underline{\textbf{97.74$\pm$1.13}} & \underline{\textbf{2.78$\pm$0.77}}
& \textbf{79.12$\pm$3.05} & \underline{\textbf{3.79$\pm$1.84}}
& \textbf{96.88$\pm$1.21} & 2.36$\pm$1.60 \\
CAL+SC~\citep{tai2026self}        
& \underline{\bf98.80$\pm$2.41} & \underline{\bf0.00$\pm$0.00}
& \underline{\bf97.23$\pm$0.49} & \underline{\bf1.56$\pm$0.76}
& \underline{\bf88.74$\pm$3.06} & \underline{\bf2.45$\pm$0.92}
& \bf96.52$\pm$2.07 & \underline{\bf0.03$\pm$0.10} \\ 
CAL+SC*~\citep{tai2026self}        
& \underline{\bf98.90$\pm$2.44} & \underline{\bf0.00$\pm$0.00}
& \underline{\bf97.64$\pm$0.76} & \bf5.02$\pm$5.86
& \underline{\bf89.87$\pm$0.98} & \underline{\bf3.11$\pm$1.11}
& \underline{\bf97.93$\pm$0.46} & 4.22$\pm$1.20 \\ 
\midrule        
SMGNN~\citep{azzolin2025beyond}              
& 99.32$\pm$0.36 & 0.50$\pm$0.92
& 97.00$\pm$0.77 & 1.90$\pm$1.01
& 84.38$\pm$2.71 & 2.07$\pm$0.62
& 98.12$\pm$0.39 & 1.72$\pm$1.09 \\
SMGNN+SD
& \textbf{99.47$\pm$0.24} & \textbf{0.30$\pm$0.64}
& \underline{\textbf{98.59$\pm$0.51}} & \underline{\textbf{0.80$\pm$0.51}}
& \underline{\textbf{88.05$\pm$1.64}} & \underline{\textbf{1.12$\pm$0.50}}
& \underline{\textbf{98.51$\pm$0.36}} & \underline{\textbf{1.28$\pm$0.93}} \\     
SMGNN+SD*
& \textbf{99.47$\pm$0.24} & \underline{\textbf{0.00$\pm$0.00}}
& \underline{\textbf{98.59$\pm$0.51}} & 4.33$\pm$5.82
& \underline{\textbf{88.05$\pm$1.64}} & \underline{\textbf{1.08$\pm$0.58}}
& \underline{\textbf{98.51$\pm$0.36}} & \textbf{1.59$\pm$0.62} \\
SMGNN+SC~\citep{tai2026self}        
& \underline{\bf99.87$\pm$0.05} & \underline{\bf0.00$\pm$0.00}
& \underline{\bf98.39$\pm$0.88} & \underline{\bf0.00$\pm$0.00}
& \underline{\bf90.07$\pm$1.56} & \underline{\bf0.97$\pm$0.59}
& \underline{\bf98.30$\pm$0.54} & 0.61$\pm$0.54 \\ 
\midrule        
GSAT~\citep{miao2022interpretable}      
& {99.30$\pm$0.47} & 0.00$\pm$0.00 
& 98.38$\pm$0.31 & 0.90$\pm$0.28
& 90.66$\pm$0.89 & 1.80$\pm$0.84 
& 99.01$\pm$0.31 & 1.11$\pm$0.61 \\          
GSAT+SD
& \textbf{99.44$\pm$0.40} & 0.00$\pm$0.00
& \underline{\textbf{99.22$\pm$0.23}} & \textbf{0.62$\pm$0.40}
& \underline{\textbf{92.17$\pm$0.71}} & \underline{\textbf{1.13$\pm$0.49}}
& \textbf{99.11$\pm$0.24} & \textbf{1.11$\pm$0.40} \\
GSAT+SD*
& \textbf{99.44$\pm$0.40} & 0.00$\pm$0.00
& \underline{\textbf{99.22$\pm$0.23}} & \underline{\textbf{0.38$\pm$0.59}}
& \underline{\textbf{92.17$\pm$0.71}} & \underline{\textbf{0.89$\pm$0.27}}
& \textbf{99.11$\pm$0.24} & 1.22$\pm$0.66 \\
GSAT+SC~\citep{tai2026self}        
& \bf99.30$\pm$0.12 & 0.00$\pm$0.00
& \underline{\bf98.91$\pm$0.29} & \underline{\bf0.31$\pm$0.29}
& \underline{\bf92.80$\pm$0.36} & \underline{\bf0.74$\pm$0.14}
& \underline{\bf99.38$\pm$0.05} & \underline{\bf0.17$\pm$0.23} \\ 
\\[-3mm] \hdashline \\[-3mm]            
\diagbox[width=2.5cm]{}{}        
& $\uparrow$ ACC (\%) & $\downarrow$ SPA (\%)      
& $\uparrow$ ACC (\%) & $\downarrow$ SPA (\%)         
& $\uparrow$ ACC (\%) & $\downarrow$ SPA (\%)         
& $\uparrow$ ACC (\%) & $\downarrow$ SPA (\%) \\ 
\midrule 
GAT~\citep{velivckovic2018graph}        
& 97.10$\pm$8.70 & 69.98$\pm$6.02  
& 96.54$\pm$1.46 & {26.31$\pm$2.25}   
& 91.60$\pm$0.66 & 64.16$\pm$4.14   
& 92.97$\pm$0.78 & {88.66$\pm$8.80} \\      
GAT+SD                                  
& \textbf{98.60$\pm$4.20} & \underline{\textbf{64.35$\pm$6.49}}  
& \textbf{97.20$\pm$1.38} & \underline{\textbf{20.13$\pm$1.75}}           
& 91.39$\pm$1.23 & \underline{\textbf{53.01$\pm$5.18}}   
& \textbf{92.94$\pm$1.02} & \underline{\textbf{74.17$\pm$21.00}} \\
GAT+SD*
& \underline{\textbf{100.00$\pm$0.00}} & \underline{\textbf{64.35$\pm$6.49}}  
& \underline{\textbf{99.07$\pm$0.41}} & \underline{\textbf{20.13$\pm$1.75}} 
& \underline{\textbf{92.66$\pm$0.97}} & \underline{\textbf{53.01$\pm$5.18}}   
& \textbf{93.28$\pm$0.63} & \underline{\textbf{74.17$\pm$21.00}} \\
GAT+SC~\citep{tai2026self}        
& \underline{\bf100.00$\pm$0.00} & 76.96$\pm$15.04
& \underline{\bf99.55$\pm$0.16} & \underline{\bf15.83$\pm$4.24}
& \underline{\bf92.46$\pm$0.32} & \underline{\bf50.94$\pm$11.00}
& \underline{\bf93.65$\pm$0.77} & 99.99$\pm$0.02 \\ 
GAT+SC*~\citep{tai2026self}        
& \underline{\bf99.80$\pm$0.60} & \underline{\bf62.75$\pm$2.69}
& \underline{\bf99.65$\pm$0.00} & \underline{\bf11.91$\pm$3.41}
& \underline{\bf92.57$\pm$0.44} & \underline{\bf35.36$\pm$3.86}
& 91.18$\pm$1.07 & \underline{\bf9.55$\pm$1.52} \\ 
\midrule        
CAL~\citep{sui2022causal}               
& 91.90$\pm$11.68 & {66.99$\pm$6.50}   
& 94.22$\pm$2.11 & {26.65$\pm$1.96}  
& 84.31$\pm$5.66 & {60.80$\pm$10.75}   
& 91.22$\pm$1.37 & 73.09$\pm$14.58    \\
CAL+SD                                  
& \underline{\textbf{92.20$\pm$11.14}} & \underline{\textbf{66.09$\pm$6.72}}  
& 93.91$\pm$2.23 & \underline{\textbf{25.30$\pm$1.78}}   
& 84.28$\pm$5.81 & \underline{\textbf{60.06$\pm$10.86}}   
& 91.15$\pm$1.34 & \underline{\textbf{67.06$\pm$17.08}}  \\
CAL+SD*
& \underline{\textbf{99.80$\pm$0.40}} & \underline{\textbf{66.09$\pm$6.72}} 
& \underline{\textbf{97.58$\pm$1.89}} & \underline{\textbf{25.30$\pm$1.78}}
& \underline{\textbf{90.04$\pm$1.12}} & \underline{\textbf{60.06$\pm$10.86}} 
& \underline{\textbf{92.84$\pm$0.88}} & \underline{\textbf{67.06$\pm$17.08}} \\
CAL+SC~\citep{tai2026self}        
& \underline{\bf100.00$\pm$0.00} & 71.17$\pm$15.59
& \underline{\bf97.82$\pm$0.97} & \underline{\bf20.57$\pm$3.24}
& \underline{\bf91.77$\pm$0.45} & \underline{\bf40.74$\pm$8.31}
& \underline{\bf93.11$\pm$0.70} & \underline{\bf50.03$\pm$0.10} \\ 
CAL+SC*~\citep{tai2026self}        
& \underline{\bf100.00$\pm$0.00} & \underline{\bf23.12$\pm$1.46}
& \underline{\bf98.27$\pm$0.93} & \underline{\bf17.06$\pm$3.81}
& \underline{\bf91.49$\pm$0.54} & \underline{\bf34.66$\pm$2.48}
& 90.88$\pm$0.71 & \underline{\bf10.11$\pm$1.45} \\ 
\midrule        
SMGNN~\citep{azzolin2025beyond}              
& 95.50$\pm$12.52 & {61.06$\pm$4.47}  
& 97.30$\pm$1.31 & {19.60$\pm$2.13}   
& 91.12$\pm$0.62 & {46.73$\pm$6.39}   
& 89.53$\pm$0.99 & 20.50$\pm$4.78    \\      
SMGNN+SD                                
& 95.50$\pm$13.50 & \underline{\textbf{54.57$\pm$4.62}}  
& \textbf{98.44$\pm$1.06} & \underline{\textbf{11.44$\pm$1.40}}  
& \textbf{91.34$\pm$1.14} & \underline{\textbf{38.46$\pm$5.41}}   
& 89.19$\pm$0.79 & \underline{\textbf{16.57$\pm$3.61}}  \\  
SMGNN+SD*
& \underline{\textbf{100.00$\pm$0.00}} & \underline{\textbf{54.57$\pm$4.62}}
& \underline{\textbf{99.03$\pm$0.57}} & \underline{\textbf{11.44$\pm$1.40}}  
& \underline{\textbf{92.07$\pm$0.79}} & \underline{\textbf{38.46$\pm$5.41}}  
& \underline{\textbf{91.18$\pm$1.24}} & \underline{\textbf{16.57$\pm$3.61}}  \\  
SMGNN+SC~\citep{tai2026self}        
& \underline{\bf99.80$\pm$0.60} & 62.75$\pm$2.69
& \underline{\bf99.65$\pm$0.00} & 11.91$\pm$3.41
& \underline{\bf92.57$\pm$0.44} & \underline{\bf35.36$\pm$3.86}
& \underline{\bf91.18$\pm$1.07} & \underline{\bf9.55$\pm$1.52} \\ 
\midrule        
GSAT~\citep{miao2022interpretable}      
& 100.00$\pm$0.00 & 74.44$\pm$2.63  
& 98.55$\pm$0.80 & 63.88$\pm$3.92   
& 91.48$\pm$0.87 & 57.10$\pm$10.18  
& 92.43$\pm$1.00 & 44.39$\pm$6.10    \\        
GSAT+SD                                 
& 100.00$\pm$0.00 & \underline{\textbf{70.24$\pm$2.92}}  
& \textbf{98.89$\pm$0.80} & \underline{\textbf{62.39$\pm$2.42}}  
& \textbf{91.72$\pm$0.92} & \underline{\textbf{54.44$\pm$8.14}}  
& 91.99$\pm$0.93 & \underline{\textbf{39.16$\pm$6.20}}   \\ 
GSAT+SD*
& 100.00$\pm$0.00 & \underline{\textbf{70.24$\pm$2.92}}  
& \underline{\textbf{99.55$\pm$0.27}} & \underline{\textbf{62.39$\pm$2.42}}  
& \underline{\textbf{92.91$\pm$0.47}} & \underline{\textbf{54.44$\pm$8.14}}  
& \underline{\textbf{93.38$\pm$0.34}} & \underline{\textbf{39.16$\pm$6.20}} \\
GSAT+SC~\citep{tai2026self}        
& 100.00$\pm$0.00 & \underline{\bf60.87$\pm$1.42}
& \underline{\bf99.55$\pm$0.16} & \underline{\bf37.60$\pm$2.73}
& \underline{\bf92.31$\pm$0.32} & 73.81$\pm$0.94
& \underline{\bf93.48$\pm$0.48} & 47.09$\pm$1.41 \\
\bottomrule
\end{tabular}
}
\end{center}
\end{table}

Recent work~\cite{tai2026self} mitigates explanation self-inconsistency by adding a self-consistency regularization term during training. We include this strategy for reference to clarify the relationship between training-time self-consistency regularization (SC) and our post-hoc SD strategy.

As shown in \cref{tab:main_results}, SC often achieves stronger explanation performance than SD, which is expected because it directly optimizes the model parameters with an additional consistency objective. However, this gain comes at the cost of retraining or fine-tuning the SI-GNN, and its effectiveness can depend on model-specific regularization settings. For example, GAT~\cite{velivckovic2018graph} and CAL~\cite{sui2022causal} require additional conciseness regularization to remain stable under self-consistency training, denoted as SC*. 

It is also worth noting that the lower FID values of SC do not necessarily imply substantially better explanations than SD. The motivation of Tai et al.~\cite{tai2026self} is faithfulness optimization, and self-consistency is studied because of its direct connection to faithfulness. Therefore, SC is naturally favored under fidelity-based evaluation. As discussed in \cref{app:metrics}, fidelity metrics can be affected by distribution shift and should be interpreted sparingly.

In contrast, SD targets a different use case: it is a training-free post-processing strategy that can be directly applied to pretrained SI-GNNs without modifying their objectives, architectures, or training procedures. From this perspective, SD is more closely aligned with EE~\cite{tai2025redundancy}, because both are post-hoc explanation calibration strategies, and we compare with EE in the main text.

\subsection{Results on Other GNN Backbones}\label{app:gnn-backbone}

In the main text, we report results using GIN as the backbone. 
To further evaluate the generality of SD, we conduct additional experiments with two other GNN backbones: GraphSAGE~\cite{hamilton2017inductive} and GatedGCN~\cite{bresson2017residual}. The results are reported in \cref{tab:main_results_graphsage}--\cref{tab:with_EE_gatedgcn}. 

Overall, SD generally improves explanation quality (plausibility, sparsity) across different backbones and is complementary with EE. After lightweight classifier adaptation, SD also improves downstream predictive performance, further demonstrating its robustness and general applicability.

\begin{table*}[t]
\caption{Experimental results of SD on SI-GNNs with GraphSAGE backbone across four datasets. \textbf{Bold} indicates better performance, while \underline{underlined} results are statistically significant ($p < 0.05$).}
\label{tab:main_results_graphsage}
\begin{center}
\setlength{\tabcolsep}{2pt}
\resizebox{\linewidth}{!}{
\begin{tabular}{ccccccccc}
\toprule
\multirow{2}{*}{\vspace{-2mm} Method}
& \multicolumn{2}{c}{BA-2MOTIFS}
& \multicolumn{2}{c}{3MR}
& \multicolumn{2}{c}{BENZENE}
& \multicolumn{2}{c}{MUTAGENICITY} \\
\cmidrule(lr){2-3}\cmidrule(lr){4-5}\cmidrule(lr){6-7}\cmidrule(lr){8-9}
& $\uparrow$ AUC (\%) & $\downarrow$ FID (\%)
& $\uparrow$ AUC (\%) & $\downarrow$ FID (\%)
& $\uparrow$ AUC (\%) & $\downarrow$ FID (\%)
& $\uparrow$ AUC (\%) & $\downarrow$ FID (\%) \\
\midrule
GAT~\citep{velivckovic2018graph} & 99.16$\pm$1.15 & {12.80$\pm$8.61} & 99.19$\pm$0.07 & 0.45$\pm$0.68 & 92.10$\pm$0.88 & 1.89$\pm$0.91 & 97.29$\pm$1.50 & 3.18$\pm$1.37 \\
GAT+SD & \textbf{99.22$\pm$1.14} & 13.80$\pm$10.83 & \underline{\textbf{99.80$\pm$0.06}} & 0.66$\pm$1.00 & \underline{\textbf{93.51$\pm$1.00}} & \textbf{1.67$\pm$0.74} & \textbf{97.88$\pm$1.38} & \textbf{2.97$\pm$1.40} \\
GAT+SD* & \textbf{99.22$\pm$1.14} & \textbf{11.40$\pm$13.74} & \underline{\textbf{99.80$\pm$0.06}} & 0.66$\pm$1.27 & \underline{\textbf{93.51$\pm$1.00}} & 2.38$\pm$0.88 & \textbf{97.88$\pm$1.38} & 3.55$\pm$1.23 \\

\midrule
CAL~\citep{sui2022causal} & 99.70$\pm$0.28 & \textbf{28.70$\pm$14.30} & 98.95$\pm$0.15 & {0.52$\pm$0.42} & 91.21$\pm$0.85 & \textbf{6.23$\pm$2.98} & 97.82$\pm$0.73 & \textbf{2.03$\pm$0.72} \\
CAL+SD & \textbf{99.73$\pm$0.31} & 36.50$\pm$21.95 & \underline{\textbf{99.70$\pm$0.08}} & 2.32$\pm$2.56 & \underline{\textbf{92.75$\pm$0.87}} & 14.81$\pm$5.80 & \underline{\textbf{98.58$\pm$0.29}} & 5.95$\pm$1.58 \\
CAL+SD* & \textbf{99.73$\pm$0.31} & 32.00$\pm$15.24 & \underline{\textbf{99.70$\pm$0.08}} & 0.97$\pm$1.68 & \underline{\textbf{92.75$\pm$0.87}} & 10.27$\pm$6.04 & \underline{\textbf{98.58$\pm$0.29}} & 10.88$\pm$4.36 \\

\midrule
SMGNN~\citep{azzolin2025beyond} & 99.72$\pm$0.25 & {28.10$\pm$14.73} & 98.41$\pm$0.61 & 1.00$\pm$1.14 & 91.04$\pm$2.17 & 3.24$\pm$1.38 & 98.39$\pm$0.49 & 2.03$\pm$0.66 \\
SMGNN+SD & \textbf{99.83$\pm$0.25} & {29.60$\pm$17.37} & \underline{\textbf{99.34$\pm$0.43}} & {0.97$\pm$1.69} & \textbf{92.34$\pm$2.55} & \textbf{2.97$\pm$1.37} & \textbf{98.82$\pm$0.46} & \underline{\textbf{1.05$\pm$0.49}} \\
SMGNN+SD* & \textbf{99.83$\pm$0.25} & 33.30$\pm$12.47 & \underline{\textbf{99.34$\pm$0.43}} & \textbf{0.76$\pm$0.89} & \textbf{92.34$\pm$2.55} & 4.76$\pm$2.80 & \textbf{98.82$\pm$0.46} & 1.62$\pm$0.78 \\

\midrule
GSAT~\citep{miao2022interpretable} & {99.10$\pm$1.06} & \textbf{1.60$\pm$2.33} & 99.41$\pm$0.21 & 0.38$\pm$0.48 & 92.72$\pm$0.79 & 1.20$\pm$0.20 & 99.16$\pm$0.10 & 0.81$\pm$0.61 \\
GSAT+SD & 99.06$\pm$1.37 & 4.10$\pm$6.32 & \underline{\textbf{99.83$\pm$0.18}} & {0.21$\pm$0.35} & \textbf{93.81$\pm$0.90} & \textbf{1.18$\pm$0.26} & \textbf{99.17$\pm$0.09} & {0.74$\pm$0.56} \\
GSAT+SD* & 99.06$\pm$1.37 & 5.10$\pm$7.58 & \underline{\textbf{99.83$\pm$0.18}} & \textbf{0.07$\pm$0.14} & \textbf{93.81$\pm$0.90} & 1.78$\pm$0.48 & \textbf{99.17$\pm$0.09} & \textbf{0.47$\pm$0.48} \\
\\[-3mm] \hdashline \\[-3mm]
\diagbox[width=2.5cm]{}{}
& $\uparrow$ ACC (\%) & $\downarrow$ SPA (\%)
& $\uparrow$ ACC (\%) & $\downarrow$ SPA (\%)
& $\uparrow$ ACC (\%) & $\downarrow$ SPA (\%)
& $\uparrow$ ACC (\%) & $\downarrow$ SPA (\%) \\
\midrule
GAT~\citep{velivckovic2018graph} & {95.10$\pm$12.74} & 49.26$\pm$4.54 & {99.52$\pm$0.70} & 35.70$\pm$4.93 & {92.31$\pm$1.34} & 46.48$\pm$17.34 & {91.32$\pm$1.31} & 30.51$\pm$10.93 \\
GAT+SD & 94.00$\pm$14.81 & \underline{\textbf{47.75$\pm$4.59}} & 99.34$\pm$1.00 & \underline{\textbf{34.40$\pm$5.10}} & \textbf{92.17$\pm$1.34} & \underline{\textbf{45.53$\pm$17.31}} & 91.28$\pm$1.23 & \underline{\textbf{29.05$\pm$10.90}} \\
GAT+SD* & \underline{\textbf{100.00$\pm$0.00}} & \underline{\textbf{47.75$\pm$4.59}} & \underline{\textbf{100.00$\pm$0.00}} & \underline{\textbf{34.40$\pm$5.10}} & \underline{\bf93.23$\pm$1.05} & \underline{\textbf{45.53$\pm$17.31}} & \underline{\textbf{92.23$\pm$0.96}} & \underline{\textbf{29.05$\pm$10.90}} \\

\midrule
CAL~\citep{sui2022causal} & {94.00$\pm$8.53} & 39.14$\pm$3.05 & 99.72$\pm$0.51 & 29.19$\pm$1.96 & {87.03$\pm$3.57} & 33.87$\pm$6.76 & {88.92$\pm$0.66} & 25.41$\pm$8.24 \\
CAL+SD & 93.70$\pm$10.01 & \underline{\textbf{37.31$\pm$2.97}} & \textbf{99.79$\pm$0.42} & \underline{\textbf{24.88$\pm$1.76}} & 86.75$\pm$3.58 & \underline{\textbf{32.98$\pm$6.57}} & 88.68$\pm$1.10 & \underline{\textbf{13.79$\pm$3.72}} \\
CAL+SD* & \underline{\textbf{98.80$\pm$2.68}} & \underline{\textbf{37.31$\pm$2.97}} & \bf99.72$\pm$0.30 & \underline{\textbf{24.88$\pm$1.76}} & \underline{\textbf{91.08$\pm$1.28}} & \underline{\textbf{32.98$\pm$6.57}} & \underline{\textbf{89.90$\pm$1.53}} & \underline{\textbf{13.79$\pm$3.72}} \\

\midrule
SMGNN~\citep{azzolin2025beyond} & {95.70$\pm$10.92} & 35.91$\pm$7.87 & 98.37$\pm$1.90 & 22.77$\pm$6.10 & {86.88$\pm$8.80} & 27.90$\pm$4.05 & {89.73$\pm$1.21} & 15.27$\pm$2.05 \\
SMGNN+SD & 94.50$\pm$13.55 & \underline{\textbf{34.65$\pm$7.55}} & {98.41$\pm$2.30} & \underline{\textbf{18.78$\pm$5.40}} & 86.49$\pm$9.24 & \underline{\textbf{27.14$\pm$4.02}} & 89.22$\pm$0.95 & \underline{\textbf{11.73$\pm$1.35}} \\
SMGNN+SD* & \underline{\textbf{99.90$\pm$0.30}} & \underline{\textbf{34.65$\pm$7.55}} & \underline{\textbf{99.79$\pm$0.35}} & \underline{\textbf{18.78$\pm$5.40}} & \underline{\textbf{91.54$\pm$1.52}} & \underline{\textbf{27.14$\pm$4.02}} & \underline{\textbf{91.86$\pm$1.58}} & \underline{\textbf{11.73$\pm$1.35}} \\

\midrule
GSAT~\citep{miao2022interpretable} & {99.20$\pm$1.47} & 58.26$\pm$5.61 & 99.62$\pm$0.24 & 37.23$\pm$4.63 & 93.62$\pm$0.29 & 54.98$\pm$3.25 & {91.45$\pm$1.04} & 38.67$\pm$0.67 \\
GSAT+SD & 97.70$\pm$5.92 & \underline{\textbf{56.85$\pm$5.63}} & {99.76$\pm$0.35} & \textbf{34.34$\pm$4.96} & {93.65$\pm$0.30} & \underline{\textbf{54.01$\pm$3.32}} & 91.39$\pm$1.01 & \underline{\textbf{38.27$\pm$0.70}} \\
GSAT+SD* & \underline{\textbf{100.00$\pm$0.00}} & \underline{\textbf{56.85$\pm$5.63}} & \underline{\textbf{99.93$\pm$0.14}} & \underline{\textbf{34.34$\pm$4.96}} & \textbf{93.72$\pm$0.20} & \underline{\textbf{54.01$\pm$3.32}} & \underline{\textbf{93.14$\pm$0.88}} & \underline{\textbf{38.27$\pm$0.70}} \\
\bottomrule
\end{tabular}
}
\end{center}
\end{table*}

\begin{table}[t]
\caption{Experimental results of combining SD with EE on SI-GNNs with GraphSAGE backbone.}
\label{tab:with_EE_graphsage}
\begin{center}
\setlength{\tabcolsep}{2pt}
\resizebox{\linewidth}{!}{
\begin{tabular}{ccccccccc}
\toprule
\multirow{2}{*}{\vspace{-2mm} Method}
& \multicolumn{2}{c}{BA-2MOTIFS}
& \multicolumn{2}{c}{3MR}
& \multicolumn{2}{c}{BENZENE}
& \multicolumn{2}{c}{MUTAGENICITY} \\
\cmidrule(lr){2-3}\cmidrule(lr){4-5}\cmidrule(lr){6-7}\cmidrule(lr){8-9}
& $\uparrow$ AUC (\%) & $\uparrow$ ACC (\%)& $\uparrow$ AUC (\%) & $\uparrow$ ACC (\%)& $\uparrow$ AUC (\%) & $\uparrow$ ACC (\%)& $\uparrow$ AUC (\%) & $\uparrow$ ACC (\%) \\
\midrule


GAT+EE & 99.48$\pm$0.15 & 99.90$\pm$0.31 & 99.39$\pm$0.05 & 100.00$\pm$0.02 & 94.25$\pm$0.31 & 93.71$\pm$0.56 & 98.78$\pm$0.28 & 92.77$\pm$0.52 \\

GAT+SD+EE & \textbf{99.58$\pm$0.13} & 99.88$\pm$0.37 & \textbf{99.96$\pm$0.01} & \textbf{100.00$\pm$0.00} & \textbf{95.08$\pm$0.34} & 93.65$\pm$0.57 & \textbf{99.09$\pm$0.25} & 92.60$\pm$0.53 \\

\midrule


CAL+EE & 99.74$\pm$0.07 & 98.67$\pm$0.77 & 99.30$\pm$0.06 & 100.00$\pm$0.03 & 92.94$\pm$0.28 & 88.61$\pm$1.17 & 98.67$\pm$0.19 & 89.64$\pm$0.75 \\

CAL+SD+EE & \textbf{99.83$\pm$0.07} & \textbf{98.86$\pm$0.56} & \textbf{99.94$\pm$0.02} & \textbf{100.00$\pm$0.00} & \textbf{93.99$\pm$0.30} & 88.43$\pm$1.22 & \textbf{99.12$\pm$0.13} & \textbf{90.12$\pm$0.89} \\

\midrule


SMGNN+EE & 99.74$\pm$0.06 & 99.41$\pm$0.49 & 99.07$\pm$0.06 & 99.67$\pm$0.45 & 93.05$\pm$0.50 & 90.10$\pm$1.89 & 98.94$\pm$0.07 & 90.37$\pm$0.76 \\

SMGNN+SD+EE & \textbf{99.91$\pm$0.03} & \textbf{99.52$\pm$0.50} & \textbf{99.85$\pm$0.08} & \textbf{99.93$\pm$0.21} & \textbf{93.77$\pm$0.59} & 89.65$\pm$2.12 & \textbf{99.22$\pm$0.04} & 90.09$\pm$0.68 \\

\midrule


GSAT+EE & 99.54$\pm$0.12 & 100.00$\pm$0.00 & 99.65$\pm$0.04 & 99.65$\pm$0.00 & 93.66$\pm$0.26 & 93.95$\pm$0.16 & 99.29$\pm$0.04 & 91.68$\pm$0.37 \\

GSAT+SD+EE & \textbf{99.58$\pm$0.14} & \textbf{100.00$\pm$0.00} & \textbf{99.99$\pm$0.01} & \textbf{99.98$\pm$0.08} & \textbf{94.51$\pm$0.31} & \textbf{93.97$\pm$0.17} & \textbf{99.29$\pm$0.03} & 91.67$\pm$0.37 \\

\bottomrule
\end{tabular}
}
\end{center}
\end{table}


\begin{table*}[t]
\caption{Experimental results of SD on SI-GNNs with GatedGCN backbone across four datasets. \textbf{Bold} indicates better performance, while \underline{underlined} results are statistically significant ($p < 0.05$).}
\label{tab:main_results_gatedgcn}
\begin{center}
\setlength{\tabcolsep}{2pt}
\resizebox{\linewidth}{!}{
\begin{tabular}{ccccccccc}
\toprule
\multirow{2}{*}{\vspace{-2mm} Method}
& \multicolumn{2}{c}{BA-2MOTIFS}
& \multicolumn{2}{c}{3MR}
& \multicolumn{2}{c}{BENZENE}
& \multicolumn{2}{c}{MUTAGENICITY} \\
\cmidrule(lr){2-3}\cmidrule(lr){4-5}\cmidrule(lr){6-7}\cmidrule(lr){8-9}
& $\uparrow$ AUC (\%) & $\downarrow$ FID (\%)
& $\uparrow$ AUC (\%) & $\downarrow$ FID (\%)
& $\uparrow$ AUC (\%) & $\downarrow$ FID (\%)
& $\uparrow$ AUC (\%) & $\downarrow$ FID (\%) \\
\midrule
GAT~\citep{velivckovic2018graph} & 99.12$\pm$0.90 & 18.50$\pm$14.98 & 99.10$\pm$0.16 & 0.17$\pm$0.23 & 90.40$\pm$2.31 & 3.19$\pm$1.61 & 97.87$\pm$0.90 & 2.91$\pm$1.47 \\
GAT+SD & \textbf{99.14$\pm$1.09} & \textbf{17.70$\pm$14.95} & \underline{\textbf{99.67$\pm$0.16}} & \textbf{0.10$\pm$0.16} & \textbf{91.64$\pm$2.58} & \textbf{2.89$\pm$1.44} & \textbf{98.32$\pm$0.80} & \textbf{2.77$\pm$1.37} \\
GAT+SD* & \textbf{99.14$\pm$1.09} & \bf{18.30$\pm$15.86} & \underline{\textbf{99.67$\pm$0.16}} & 0.52$\pm$1.34 & \textbf{91.64$\pm$2.58} & 3.99$\pm$1.30 & \textbf{98.32$\pm$0.80} & 3.58$\pm$1.69 \\

\midrule
CAL~\citep{sui2022causal} & 99.81$\pm$0.08 & {30.90$\pm$19.56} & 98.97$\pm$0.20 & 0.45$\pm$0.31 & 90.68$\pm$0.90 & 5.84$\pm$1.53 & 98.07$\pm$0.38 & \textbf{3.38$\pm$1.17} \\
CAL+SD & \underline{\textbf{99.88$\pm$0.09}} & 31.20$\pm$20.62 & \underline{\textbf{99.60$\pm$0.18}} & 9.79$\pm$12.79 & \underline{\textbf{91.89$\pm$1.06}} & 12.29$\pm$4.43 & \underline{\textbf{98.44$\pm$0.73}} & 6.15$\pm$1.73 \\
CAL+SD* & \underline{\textbf{99.88$\pm$0.09}} & \textbf{29.60$\pm$18.21} & \underline{\textbf{99.60$\pm$0.18}} & 2.35$\pm$3.58 & \underline{\textbf{91.89$\pm$1.06}} & 10.45$\pm$3.84 & \underline{\textbf{98.44$\pm$0.73}} & 7.20$\pm$3.18 \\

\midrule
SMGNN~\citep{azzolin2025beyond} & {98.23$\pm$2.44} & 26.70$\pm$17.86 & 98.78$\pm$0.17 & 0.90$\pm$0.75 & 88.88$\pm$3.84 & 4.48$\pm$2.12 & 98.50$\pm$0.30 & 1.66$\pm$0.46 \\
SMGNN+SD & 98.02$\pm$2.99 & \bf{26.50$\pm$16.52} & \underline{\textbf{99.51$\pm$0.17}} & \textbf{0.35$\pm$0.71} & \textbf{89.97$\pm$4.17} & \textbf{3.87$\pm$1.77} & \underline{\textbf{98.91$\pm$0.26}} & \textbf{1.39$\pm$0.44} \\
SMGNN+SD* & 98.02$\pm$2.99 & \textbf{26.20$\pm$17.55} & \underline{\textbf{99.51$\pm$0.17}} & \bf0.45$\pm$0.31 & \textbf{89.97$\pm$4.17} & 6.83$\pm$4.70 & \underline{\textbf{98.91$\pm$0.26}} & 1.99$\pm$0.78 \\

\midrule
GSAT~\citep{miao2022interpretable} & {98.49$\pm$2.62} & {1.10$\pm$2.47} & 99.52$\pm$0.10 & 0.28$\pm$0.14 & 92.65$\pm$1.01 & {1.08$\pm$0.18} & 99.21$\pm$0.08 & 0.64$\pm$0.32 \\
GSAT+SD & 98.47$\pm$2.82 & 1.70$\pm$2.93 & \underline{\textbf{99.92$\pm$0.06}} & \underline{\textbf{0.00$\pm$0.00}} & \textbf{92.97$\pm$1.53} & \textbf{0.74$\pm$0.48} & \textbf{99.22$\pm$0.07} & \textbf{0.54$\pm$0.34} \\
GSAT+SD* & 98.47$\pm$2.82 & \textbf{1.10$\pm$1.81} & \underline{\textbf{99.92$\pm$0.06}} & \underline{\textbf{0.00$\pm$0.00}} & \textbf{92.97$\pm$1.53} & 1.16$\pm$0.46 & \textbf{99.22$\pm$0.07} & \textbf{0.54$\pm$0.31} \\

\\[-3mm] \hdashline \\[-3mm]
\diagbox[width=2.5cm]{}{}
& $\uparrow$ ACC (\%) & $\downarrow$ SPA (\%)
& $\uparrow$ ACC (\%) & $\downarrow$ SPA (\%)
& $\uparrow$ ACC (\%) & $\downarrow$ SPA (\%)
& $\uparrow$ ACC (\%) & $\downarrow$ SPA (\%) \\
\midrule
GAT~\citep{velivckovic2018graph} & {98.90$\pm$2.34} & 51.47$\pm$7.62 & 99.58$\pm$0.48 & 32.96$\pm$2.90 & {90.24$\pm$3.07} & 37.19$\pm$12.59 & {90.57$\pm$1.96} & 28.49$\pm$9.36 \\
GAT+SD & 98.60$\pm$3.23 & \underline{\textbf{46.96$\pm$7.91}} & {99.65$\pm$0.49} & \underline{\textbf{31.78$\pm$2.68}} & 90.02$\pm$3.46 & \underline{\textbf{36.11$\pm$12.58}} & 90.54$\pm$1.88 & \underline{\textbf{27.10$\pm$9.28}} \\
GAT+SD* & \underline{\textbf{100.00$\pm$0.00}} & \underline{\textbf{46.96$\pm$7.91}} & \underline{\textbf{99.93$\pm$0.14}} & \underline{\textbf{31.78$\pm$2.68}} & \underline{\textbf{91.71$\pm$1.58}} & \underline{\textbf{36.11$\pm$12.58}} & \underline{\textbf{92.60$\pm$1.03}} & \underline{\textbf{27.10$\pm$9.28}} \\

\midrule
CAL~\citep{sui2022causal} & {99.00$\pm$1.00} & 38.37$\pm$2.69 & 99.31$\pm$0.41 & 27.64$\pm$4.10 & {87.16$\pm$2.01} & 33.85$\pm$7.54 & {87.77$\pm$1.22} & 23.77$\pm$5.76 \\
CAL+SD & 98.90$\pm$1.14 & \underline{\textbf{36.62$\pm$2.75}} & \textbf{99.38$\pm$0.48} & \underline{\textbf{23.48$\pm$3.58}} & 87.03$\pm$2.09 & \underline{\textbf{32.89$\pm$7.43}} & \bf{88.62$\pm$2.23} & \underline{\textbf{11.38$\pm$2.23}} \\
CAL+SD* & 98.30$\pm$5.10 & \underline{\textbf{36.62$\pm$2.75}} & 99.03$\pm$0.65 & \underline{\textbf{23.48$\pm$3.58}} & \underline{\textbf{90.34$\pm$1.03}} & \underline{\textbf{32.89$\pm$7.43}} & \underline{\textbf{89.97$\pm$1.11}} & \underline{\textbf{11.38$\pm$2.23}} \\

\midrule
SMGNN~\citep{azzolin2025beyond} & {91.90$\pm$16.03} & 37.46$\pm$4.68 & 98.48$\pm$1.17 & 20.85$\pm$4.02 & {82.66$\pm$9.47} & 25.40$\pm$4.97 & 89.93$\pm$1.04 & 15.25$\pm$2.57 \\
SMGNN+SD & {91.30$\pm$16.03} & \underline{\textbf{36.14$\pm$9.63}} & \bf{99.07$\pm$1.08} & \underline{\textbf{14.64$\pm$3.24}} & 81.81$\pm$10.19 & \underline{\textbf{24.61$\pm$4.95}} & {89.97$\pm$1.09} & \underline{\textbf{14.46$\pm$2.48}} \\
SMGNN+SD* & \underline{\textbf{99.50$\pm$1.50}} & \underline{\bf36.14$\pm$9.63} & \underline{\textbf{99.97$\pm$0.10}} & \underline{\textbf{14.64$\pm$3.24}} & \underline{\bf90.89$\pm$1.97} & \underline{\textbf{24.61$\pm$4.95}} & \underline{\bf91.76$\pm$1.26} & \underline{\bf14.46$\pm$2.48} \\

\midrule
GSAT~\citep{miao2022interpretable} & {100.00$\pm$0.00} & 57.33$\pm$5.22 & 99.72$\pm$0.14 & 35.43$\pm$2.37 & 93.43$\pm$0.36 & 57.57$\pm$3.47 & {91.42$\pm$1.36} & 39.54$\pm$1.50 \\
GSAT+SD & 99.90$\pm$0.30 & \underline{\textbf{56.13$\pm$5.24}} & \underline{\textbf{100.00$\pm$0.00}} & \underline{\textbf{29.80$\pm$2.36}} & {93.14$\pm$0.28} & \underline{\textbf{48.54$\pm$4.24}} & 91.39$\pm$1.37 & \underline{\textbf{39.16$\pm$1.50}} \\
GSAT+SD* & {100.00$\pm$0.00} & \underline{\bf56.13$\pm$5.24} & \underline{\bf100.00$\pm$0.00} & \underline{\textbf{29.80$\pm$2.36}} & \textbf{93.63$\pm$0.28} & \underline{\bf48.54$\pm$4.24} & \underline{\bf92.91$\pm$0.77} & \underline{\textbf{39.16$\pm$1.50}} \\

\bottomrule
\end{tabular}
}
\end{center}
\end{table*}

\begin{table*}[t]
\caption{Experimental results of combining SD with EE on SI-GNNs with GatedGCN backbone.}
\label{tab:with_EE_gatedgcn}
\begin{center}
\setlength{\tabcolsep}{2pt}
\resizebox{\linewidth}{!}{
\begin{tabular}{ccccccccc}
\toprule
\multirow{2}{*}{\vspace{-2mm} Method}
& \multicolumn{2}{c}{BA-2MOTIFS}
& \multicolumn{2}{c}{3MR}
& \multicolumn{2}{c}{BENZENE}
& \multicolumn{2}{c}{MUTAGENICITY} \\
\cmidrule(lr){2-3}\cmidrule(lr){4-5}\cmidrule(lr){6-7}\cmidrule(lr){8-9}
& $\uparrow$ AUC (\%) & $\uparrow$ ACC (\%)& $\uparrow$ AUC (\%) & $\uparrow$ ACC (\%)& $\uparrow$ AUC (\%) & $\uparrow$ ACC (\%)& $\uparrow$ AUC (\%) & $\uparrow$ ACC (\%) \\
\midrule


GAT+EE & 99.53$\pm$0.11 & 99.97$\pm$0.16 & 99.34$\pm$0.03 & 99.83$\pm$0.22 & 92.54$\pm$0.61 & {91.68$\pm$1.17} & 98.82$\pm$0.27 & {92.88$\pm$0.60} \\

GAT+SD+EE & \textbf{99.64$\pm$0.11} & \textbf{99.98$\pm$0.14} & \textbf{99.90$\pm$0.04} & \textbf{99.88$\pm$0.20} & \textbf{93.45$\pm$0.68} & 91.37$\pm$1.28 & \textbf{99.09$\pm$0.23} & 92.87$\pm$0.59 \\

\midrule


CAL+EE & 99.83$\pm$0.03 & 99.48$\pm$0.50 & 99.21$\pm$0.08 & 99.72$\pm$0.24 & 92.45$\pm$0.40 & {88.93$\pm$0.89} & 98.86$\pm$0.28 & 88.50$\pm$0.73 \\

CAL+SD+EE & \textbf{99.92$\pm$0.03} & \textbf{99.57$\pm$0.53} & \textbf{99.86$\pm$0.06} & \textbf{99.83$\pm$0.17} & \textbf{93.39$\pm$0.41} & 88.89$\pm$0.92 & \textbf{98.90$\pm$0.28} & \textbf{88.67$\pm$0.80} \\

\midrule


SMGNN+EE & 99.26$\pm$0.29 & 98.86$\pm$0.95 & 99.16$\pm$0.05 & 99.25$\pm$0.44 & 91.98$\pm$0.97 & {84.85$\pm$4.18} & 98.88$\pm$0.10 & {90.74$\pm$0.75} \\

SMGNN+SD+EE & \textbf{99.54$\pm$0.23} & \textbf{98.97$\pm$0.96} & \textbf{99.84$\pm$0.04} & \textbf{99.64$\pm$0.34} & \textbf{92.54$\pm$1.03} & 83.90$\pm$4.60 & \textbf{99.14$\pm$0.05} & 90.69$\pm$0.73 \\

\midrule


GSAT+EE & 99.52$\pm$0.20 & {100.00$\pm$0.00} & 99.71$\pm$0.03 & 99.65$\pm$0.00 & 93.73$\pm$0.31 & 93.67$\pm$0.20 & 99.32$\pm$0.02 & {91.98$\pm$0.47} \\

GSAT+SD+EE & \textbf{99.55$\pm$0.21} & \textbf{100.00$\pm$0.00} & \textbf{99.99$\pm$0.00} & \textbf{100.00$\pm$0.00} & \textbf{94.55$\pm$0.32} & \textbf{93.74$\pm$0.21} & \textbf{99.34$\pm$0.02} & 91.97$\pm$0.46 \\

\bottomrule
\end{tabular}
}
\end{center}
\end{table*}

\section{Limitations}\label{app:limitation}

Our latent signal assignment hypothesis should be viewed as a mechanistic perspective rather than a directly identifiable latent-variable model. The proposed positive-signal, negative-signal, and context-driven edge categories are inferred indirectly through their observed behaviors under re-explanation, rather than explicitly estimated or uniquely identifiable from model parameters. As a result, our analysis does not constitute a rigorous causal proof of why self-inconsistency arises. Instead, it provides an explanatory framework supported by theoretical analysis, empirical observations, and the effectiveness of the resulting calibration strategy. More broadly, understanding why neural networks produce unstable or unreliable explanations remains fundamentally challenging due to the highly distributed and non-transparent nature of learned representations. Similar to many recent studies on emergent behaviors in modern deep learning systems, our work relies on combining empirical phenomena, theoretical abstractions, and practical validation to develop plausible mechanistic interpretations. We hope this work can serve as a step toward more rigorous understanding and reliability analysis of self-interpretable graph learning systems.

SD leverages self-inconsistency for denoising. While both single-model self-inconsistency (as used in SD) and cross-model inconsistency (as used in EE) provide useful signals, they are inherently incomplete and cannot capture all noise sources. Thus, although SD (and its combination with EE) improves explanation quality, it does not guarantee fully trustworthy explanations. This limitation motivates future work on new criteria for assessing GNN explanation trustworthiness.

\section{Ethics Statement \& Broader Impacts}\label{app:impacts}

This work uses only publicly available benchmark datasets for graph explanation tasks. No human subjects, personally identifiable information, or sensitive data are involved. Our work aims to improve explanation quality, which has positive implications for transparency and trustworthiness in AI systems. Trustworthy explanations are particularly important in risk-sensitive domains where graph learning models may support human decisions, such as scientific discovery, drug discovery, healthcare, and finance. In these settings, misleading explanations can affect downstream interpretation and decision-making, while more trustworthy explanations can help users better assess model behavior and avoid over-interpreting spurious patterns. By identifying and suppressing context-driven explanation noise, our method may contribute to safer and more reliable use of GNNs in such applications.




\end{document}